\documentclass[11pt]{article}
\usepackage{graphicx} % Required for inserting images

\usepackage{amsmath}
\usepackage{amssymb}
\usepackage{amsthm}
\usepackage{mathtools}
\usepackage{thmtools}
\usepackage{thm-restate}
\usepackage{booktabs}
\usepackage{multirow}
\usepackage{subcaption}
\usepackage{makecell}

\usepackage{times}
\usepackage{authblk}
\usepackage{fullpage}

\usepackage{algorithm}
\usepackage{algorithmic}

\usepackage{ifthen}
\newboolean{isSingleColumn}
\makeatletter
\if@twocolumn
    \setboolean{isSingleColumn}{false}
\else
    \setboolean{isSingleColumn}{true}
\fi

\usepackage{complexity}
\let\E\relax

\usepackage{xcolor}

\usepackage{tikz}
\usetikzlibrary{shapes, arrows.meta, positioning, fit, backgrounds, decorations.pathreplacing, calc, shadows}

\definecolor{patientblue}{RGB}{70, 130, 180}
\definecolor{donorred}{RGB}{220, 80, 80}
\definecolor{clustergray}{RGB}{240, 240, 245}
\definecolor{edgegreen}{RGB}{46, 139, 87}

\tikzset{
    complexheart/.pic={
        \begin{scope}[yscale=-1, shift={(-4.5, -9)}]
            \draw[fill=red!30!white](.456,3.236)
                .. controls (.422,4.168) and (.408,5.095) .. (.461,6.046)
                .. controls (.475,6.228) and (.365,6.400) .. (.379,6.601)
                -- (.819,11.816)
                .. controls (.843,12.194) and (.838,13.389) .. (.900,13.972)
                .. controls (.943,14.340) and (2.870,14.340) .. (2.903,13.972)
                .. controls (2.927,13.699) and (2.932,13.436) .. (2.903,13.169)
                -- (1.847,5.401)
                .. controls (1.914,4.627) and (2.033,3.924) .. (2.129,3.193)
                .. controls (2.177,2.906) and (.441,2.863) .. (.456,3.236) ;
            \draw[fill=red!30!white] (.456,3.236)
                .. controls (.422,4.168) and (.408,5.095) .. (.461,6.046)
                .. controls (.475,6.228) and (.365,6.400) .. (.379,6.601)
                -- (.819,11.816)
                .. controls (.843,12.194) and (.838,13.389) .. (.900,13.972)
                .. controls (.943,14.340) and (2.870,14.340) .. (2.903,13.972)
                .. controls (2.927,13.699) and (2.932,13.436) .. (2.903,13.169)
                -- (1.847,5.401)
                .. controls (1.914,4.627) and (2.033,3.924) .. (2.129,3.193)
                .. controls (2.177,2.906) and (.441,2.863) .. (.456,3.236) ;
            \draw[fill=red!60!white] (.303,6.697)
                .. controls (-.161,7.442) and (-.022,9.641) .. (.484,11.118)
                .. controls (.795,12.022) and (1.297,12.414) .. (2.062,12.801)
                .. controls (3.429,13.513) and (5.112,13.714) .. (6.598,13.795)
                .. controls (9.495,13.957) and (9.738,11.735) .. (8.974,8.967)
                .. controls (8.744,8.126) and (8.123,7.815) .. (7.917,7.222)
                .. controls (7.683,6.606) and (7.310,5.999) .. (6.598,5.420)
                .. controls (5.570,4.579) and (1.230,5.210) .. (.303,6.697) ;
            \fill[fill=red!25!white] (2.540,6.405)
                .. controls (2.296,6.257) and (1.890,6.300) .. (1.608,6.424)
                .. controls (1.274,6.577) and (1.101,6.802) .. (1.001,7.151)
                .. controls (.858,7.652) and (1.025,8.293) .. (1.106,8.795)
                .. controls (1.149,9.072) and (1.149,9.402) .. (1.254,9.684)
                .. controls (1.278,9.746) and (1.278,9.823) .. (1.360,9.870)
                .. controls (1.426,9.909) and (1.412,9.732) .. (1.460,9.651)
                .. controls (1.747,9.144) and (2.210,8.346) .. (2.502,7.834)
                .. controls (2.602,7.657) and (2.631,7.437) .. (2.531,7.303)
                .. controls (2.411,7.146) and (2.206,6.883) .. (2.306,6.768)
                .. controls (2.382,6.682) and (2.621,6.453) .. (2.540,6.405) ;
            \fill[fill=red!25!white] (4.204,6.606)
                .. controls (5.413,6.434) and (5.327,7.543) .. (5.466,7.796)
                .. controls (5.599,8.054) and (6.350,8.785) .. (6.498,9.029)
                .. controls (6.646,9.273) and (7.282,9.947) .. (7.067,10.616)
                .. controls (6.847,11.290) and (5.776,11.582) .. (5.212,11.429)
                .. controls (4.643,11.271) and (3.773,10.941) .. (3.166,11.113)
                .. controls (2.559,11.280) and (1.350,10.793) .. (1.790,10.052)
                .. controls (2.225,9.311) and (3.640,6.692) .. (4.204,6.606) ;
            \draw[fill=yellow] (1.307,3.107)
                .. controls (1.675,3.107) and (1.976,3.169) .. (1.976,3.250)
                .. controls (1.976,3.327) and (1.675,3.394) .. (1.307,3.394)
                .. controls (.934,3.394) and (.633,3.327) .. (.633,3.250)
                .. controls (.633,3.169) and (.934,3.107) .. (1.307,3.107) ;
            \draw[fill=blue!20!red!80!white] (1.307,3.107)
                .. controls (1.675,3.107) and (1.976,3.169) .. (1.976,3.250)
                .. controls (1.976,3.327) and (1.675,3.394) .. (1.307,3.394)
                .. controls (.934,3.394) and (.633,3.327) .. (.633,3.250)
                .. controls (.633,3.169) and (.934,3.107) .. (1.307,3.107) ;
            \fill[fill=red!10!white] (1.024,14.000)
                .. controls (.991,13.929) and (.924,12.146) .. (1.039,12.323)
                .. controls (1.120,12.452) and (1.498,12.753) .. (1.870,12.767)
                .. controls (2.004,12.777) and (1.918,14.129) .. (1.875,14.163)
                .. controls (1.789,14.230) and (1.110,14.173) .. (1.024,14.000) ;
            \fill[fill=red!10!white] (.618,3.432)
                -- (.618,6.247)
                .. controls (.852,5.941) and (1.086,5.745) .. (1.345,5.578)
                .. controls (1.364,5.033) and (1.368,4.125) .. (1.345,3.556)
                .. controls (1.096,3.556) and (.862,3.518) .. (.618,3.432) ;
            \draw[fill=white] (4.872,5.215)
                .. controls (4.944,5.607) and (4.891,5.884) .. (4.595,6.061)
                .. controls (4.389,6.185) and (3.749,6.013) .. (3.864,5.664)
                .. controls (3.974,5.339) and (4.088,5.081) .. (4.346,4.923)
                .. controls (4.533,4.813) and (4.858,5.119) .. (4.872,5.215) ;
            
            \draw[fill=red!80!blue] (5.662,5.665)
                .. controls (5.643,5.732) and (5.624,5.803) .. (5.604,5.880)
                .. controls (5.944,6.076) and (6.355,6.640) .. (6.522,7.314)
                .. controls (6.699,8.031) and (6.336,8.284) .. (5.853,8.796)
                .. controls (5.523,9.154) and (5.313,9.541) .. (5.303,10.014)
                .. controls (5.485,9.106) and (5.844,8.958) .. (6.030,8.877)
                .. controls (6.278,8.767) and (6.221,9.221) .. (6.159,10.029)
                .. controls (6.341,9.412) and (6.298,8.944) .. (6.274,8.752)
                .. controls (6.245,8.561) and (6.589,8.370) .. (6.685,8.093)
                .. controls (6.799,7.777) and (6.924,8.069) .. (6.991,8.208)
                .. controls (7.153,8.547) and (7.311,8.681) .. (7.263,8.987)
                .. controls (7.096,10.033) and (6.780,10.387) .. (6.250,10.650)
                .. controls (6.728,10.492) and (6.833,10.373) .. (6.976,10.230)
                .. controls (6.991,10.440) and (6.900,10.870) .. (6.962,11.209)
                .. controls (6.957,10.669) and (7.277,9.197) .. (7.469,9.044)
                .. controls (7.579,8.953) and (8.296,9.632) .. (8.372,10.139)
                .. controls (8.458,10.698) and (8.936,11.535) .. (8.716,12.213)
                .. controls (8.970,11.458) and (8.535,10.636) .. (8.635,10.445)
                .. controls (8.778,10.641) and (9.060,11.037) .. (9.027,11.257)
                .. controls (8.994,11.477) and (9.170,11.874) .. (9.094,12.333)
                .. controls (9.228,11.883) and (9.065,11.702) .. (9.118,11.377)
                .. controls (9.166,11.052) and (8.979,10.626) .. (8.836,10.359)
                .. controls (8.549,9.828) and (7.889,9.187) .. (7.373,8.595)
                .. controls (7.708,8.796) and (8.535,9.077) .. (8.487,9.369)
                .. controls (8.434,9.661) and (8.860,9.900) .. (9.027,10.306)
                .. controls (8.903,9.928) and (8.587,9.656) .. (8.568,9.321)
                .. controls (8.554,8.982) and (8.023,8.815) .. (7.889,8.762)
                .. controls (7.995,8.752) and (8.253,8.748) .. (8.348,8.834)
                .. controls (8.444,8.920) and (8.702,9.097) .. (8.912,9.044)
                .. controls (8.539,9.044) and (8.463,8.671) .. (8.286,8.657)
                .. controls (8.109,8.647) and (7.641,8.566) .. (7.531,8.470)
                .. controls (7.015,8.016) and (6.599,7.223) .. (6.819,7.271)
                .. controls (7.344,7.381) and (7.899,7.471) .. (8.377,7.988)
                .. controls (8.023,7.567) and (7.999,7.562) .. (7.478,7.314)
                .. controls (7.091,7.137) and (6.537,6.979) .. (6.355,6.511)
                .. controls (6.250,6.229) and (5.887,5.789) .. (5.662,5.665) ;
            \draw[fill=red!80!blue] (3.387,5.951)
                .. controls (3.779,7.003) and (2.665,8.600) .. (1.857,9.866)
                .. controls (1.384,10.607) and (1.704,11.157) .. (2.364,11.305)
                .. controls (2.842,11.410) and (4.515,11.004) .. (4.993,11.023)
                .. controls (5.480,11.037) and (5.604,10.856) .. (5.882,10.559)
                .. controls (5.815,10.894) and (5.256,11.305) .. (4.997,11.209)
                .. controls (4.735,11.114) and (4.199,11.229) .. (3.975,11.338)
                .. controls (4.280,11.425) and (4.787,11.401) .. (5.179,11.601)
                .. controls (5.571,11.802) and (5.920,11.750) .. (6.159,11.300)
                .. controls (6.140,11.730) and (5.538,11.955) .. (5.126,11.783)
                .. controls (4.715,11.616) and (3.736,11.434) .. (3.401,11.434)
                .. controls (3.176,11.429) and (3.009,11.468) .. (2.617,11.563)
                .. controls (3.324,11.635) and (4.252,12.084) .. (5.165,12.304)
                .. controls (5.863,12.471) and (6.680,12.452) .. (6.790,12.189)
                .. controls (6.895,11.946) and (7.072,11.840) .. (7.473,11.793)
                .. controls (7.000,12.008) and (6.866,12.156) .. (6.814,12.438)
                .. controls (7.015,12.490) and (7.411,12.256) .. (7.626,12.457)
                .. controls (7.846,12.658) and (8.439,12.653) .. (8.864,12.381)
                .. controls (8.386,12.730) and (7.770,12.710) .. (7.607,12.557)
                .. controls (7.445,12.405) and (7.077,12.553) .. (6.862,12.591)
                .. controls (6.647,12.634) and (6.446,12.644) .. (5.891,12.548)
                .. controls (5.714,12.620) and (5.657,12.715) .. (5.867,12.792)
                .. controls (6.082,12.835) and (6.537,12.806) .. (7.014,12.806)
                .. controls (7.258,12.806) and (7.430,12.935) .. (7.631,12.811)
                .. controls (7.832,12.682) and (8.033,12.873) .. (8.324,13.016)
                .. controls (7.999,12.906) and (7.746,12.768) .. (7.660,12.897)
                .. controls (7.574,13.031) and (7.263,12.916) .. (7.163,12.940)
                .. controls (7.292,12.992) and (7.493,12.935) .. (7.507,13.136)
                .. controls (7.526,13.337) and (8.057,13.284) .. (7.999,13.652)
                .. controls (7.923,13.294) and (7.540,13.423) .. (7.406,13.179)
                .. controls (7.277,12.940) and (6.885,13.055) .. (6.680,13.031)
                .. controls (6.250,12.973) and (5.557,13.074) .. (5.461,12.639)
                .. controls (5.217,12.600) and (4.921,12.514) .. (4.682,12.419)
                .. controls (4.687,12.538) and (5.103,13.026) .. (5.461,13.102)
                .. controls (5.815,13.174) and (6.532,13.155) .. (6.699,13.681)
                .. controls (6.345,13.174) and (5.614,13.356) .. (5.294,13.260)
                .. controls (4.978,13.160) and (4.213,12.285) .. (3.898,12.189)
                .. controls (3.583,12.099) and (3.272,11.922) .. (3.157,11.989)
                .. controls (3.062,12.084) and (3.119,12.600) .. (3.583,12.581)
                .. controls (4.046,12.562) and (4.362,12.882) .. (4.620,13.356)
                .. controls (4.271,12.897) and (3.721,12.615) .. (3.358,12.744)
                .. controls (2.990,12.868) and (2.823,11.936) .. (2.598,11.812)
                .. controls (2.373,11.683) and (1.594,11.377) .. (1.503,11.042)
                .. controls (1.465,11.377) and (2.106,11.903) .. (2.043,12.677)
                .. controls (1.953,11.922) and (1.522,11.783) .. (1.317,11.377)
                .. controls (1.111,10.966) and (1.212,10.555) .. (1.417,9.991)
                .. controls (1.422,9.971) and (3.836,6.267) .. (2.588,5.985)
                .. controls (2.235,5.904) and (1.981,5.813) .. (1.757,5.660)
                .. controls (1.871,4.489) and (2.187,2.481) .. (2.650,2.405)
                .. controls (3.176,2.314) and (4.185,2.878) .. (4.175,3.289)
                .. controls (4.166,3.724) and (3.817,4.766) .. (3.994,5.201)
                .. controls (3.759,5.397) and (3.277,5.617) .. (3.387,5.951) ;
            \fill[fill=red!90!blue!70!white] (2.942,5.378)
                .. controls (2.875,5.339) and (2.736,5.378) .. (2.722,5.459)
                .. controls (2.712,5.535) and (2.741,5.583) .. (2.779,5.674)
                .. controls (3.013,6.066) and (3.200,6.716) .. (2.980,7.266)
                .. controls (2.803,7.720) and (2.564,8.226) .. (2.296,8.661)
                .. controls (2.048,9.068) and (1.627,9.670) .. (1.565,10.363)
                .. controls (1.579,10.105) and (1.837,9.737) .. (1.952,9.512)
                .. controls (2.263,8.905) and (3.071,7.815) .. (3.214,7.180)
                .. controls (3.362,6.534) and (3.243,5.875) .. (3.071,5.597)
                .. controls (3.066,5.540) and (3.028,5.425) .. (2.942,5.378) ;
            \fill[fill=red!90!blue!70!white] (2.793,2.424)
                .. controls (2.253,3.542) and (2.115,3.948) .. (2.014,5.712)
                .. controls (2.019,5.784) and (2.349,5.903) .. (2.339,5.841)
                .. controls (2.301,4.665) and (2.650,3.289) .. (3.138,2.486)
                .. controls (3.032,2.447) and (2.870,2.419) .. (2.793,2.424) ;
            \draw (2.526,6.104)
                .. controls (2.990,5.937) and (2.541,6.688) .. (2.545,6.673)
                .. controls (2.454,6.831) and (2.521,6.869) .. (2.564,6.922)
                .. controls (2.717,7.123) and (2.932,7.165) .. (2.660,7.471) ;
            \draw (.461,3.241)
                .. controls (.576,3.657) and (2.187,3.657) .. (2.129,3.198) ;
            \draw[fill=red!30!white] (3.687,5.731)
                .. controls (3.520,6.229) and (5.541,6.616) .. (5.556,6.214)
                .. controls (5.570,5.827) and (5.790,5.325) .. (5.800,4.919)
                .. controls (5.809,4.584) and (4.480,4.250) .. (4.308,4.517)
                .. controls (4.084,4.866) and (3.854,5.249) .. (3.687,5.731) ;
            \draw (4.299,4.531)
                .. controls (4.189,4.885) and (5.752,5.315) .. (5.800,4.961) ;
            \draw[fill=blue!20!red!80!white] (5.097,4.594)
                .. controls (5.436,4.694) and (5.699,4.823) .. (5.680,4.919)
                .. controls (5.656,5.009) and (5.360,5.000) .. (5.020,4.899)
                .. controls (4.681,4.799) and (4.428,4.675) .. (4.452,4.584)
                .. controls (4.475,4.493) and (4.757,4.493) .. (5.097,4.594) ;
        \end{scope}
    }
}

\usepackage{hyperref}

\usepackage{natbib}
\usepackage{enumitem}

\usepackage[capitalize,noabbrev]{cleveref}

\newtheorem{theorem}{Theorem}

\theoremstyle{definition}

\newtheorem{assumption}[theorem]{Assumption}

\crefname{assumption}{Assumption}{Assumptions}
\Crefname{assumption}{Assumption}{Assumptions}

\theoremstyle{remark}
\newtheorem{remark}[theorem]{Remark}

\newcommand{\N}{\mathbb{N}}
\newcommand{\E}{\mathbb{E}}
\newcommand{\pr}{\mathbb{P}}
\newcommand{\defeq}{\coloneqq}
\newcommand{\opt}{\mathrm{OPT}}

\newcommand{\alg}{\mathsf{ALG}}
\newcommand{\alggood}{\mathsf{ALG}_{\text{good}}}
\newcommand{\algbad}{\mathsf{ALG}_{\text{bad}}}
\newcommand{\Ugood}{\mathcal{U}_{\text{good}}}
\newcommand{\Ubad}{\mathcal{U}_{\text{bad}}}

\newcommand{\PSI}{\mathsf{PSI}}
\newcommand{\HCR}{\mathsf{HCR}}

\newcommand{\baralggood}{\overline{\alg}_{\text{good}}}
\newcommand{\baralgbad}{\overline{\alg}_{\text{bad}}}

\newcommand{\optgood}{\mathrm{OPT}_{\text{good}}}
\newcommand{\optbad}{\mathrm{OPT}_{\text{bad}}}

\newcommand{\baroptgood}{\overline{\mathrm{OPT}}_{\text{good}}}

\title{Near-Optimal Dynamic Matching via Coarsening with Application to Heart Transplantation}

\author[1]{Itai Zilberstein\protect\footnotemark[1]}
\author[1]{Ioannis Anagnostides\thanks{Equal contribution. Correspondence to \texttt{sandholm@cs.cmu.edu}}}
\author[2]{Zachary W. Sollie}
\author[2]{Arman Kilic}
\author[1,3]{Tuomas Sandholm}

\affil[1]{Department of Computer Science, Carnegie Mellon University, Pittsburgh, PA}
\affil[2]{Department of Surgery, Division of Cardiothoracic Surgery, Medical University of South Carolina, Charleston, SC}
\affil[3]{Additional affiliations: Strategy Robot, Inc., Strategic Machine, Inc., Optimized Markets, Inc.}

\begin{document}

\maketitle
\thispagestyle{empty}

\begin{abstract}
    Online matching has been a mainstay in domains such as Internet advertising and organ allocation, but practical algorithms often lack strong theoretical guarantees. We take an important step toward addressing this by developing new online matching algorithms based on a \emph{coarsening} approach. Although coarsening typically implies a loss of granularity, we show that, to the contrary, aggregating offline nodes into capacitated clusters can yield near-optimal theoretical guarantees. We apply our methodology to heart transplant allocation to develop theoretically grounded policies based on structural properties of historical data. Furthermore, in simulations based on real data, our policy closely matches the performance of the omniscient benchmark, achieving competitive ratio 0.91, drastically higher than the US \emph{status quo} policy's 0.51. Our work bridges the gap between data-driven heuristics and pessimistic theoretical lower bounds.
\end{abstract}

%Specifically, we identify a key structural property in the historical data---the presence of large patient groups with similar transplant benefits---that renders our coarsening approach highly effective. 

%Old version

%The current policy for heart transplant allocation in the US fails to integrate more granular estimates of pretransplant mortality risk and post-transplant survival prospects. While many promising data-driven approaches have been put forward in recent years, developing improved policies with provable guarantees is challenging due to the inherent online nature of the allocation process. In this paper, we bridge the gap between data-driven heuristics and pessimistic theoretical algorithms. We identify a structural property in the historical data---the presence of large patient clusters with similar transplant benefits---and leverage this property to develop a near-optimal allocation policy \emph{vis-\`a-vis} the omniscient benchmark endowed with full foresight. Our results provide rigorous justification for prior data-driven approaches in organ allocation based on clustering.
\newpage
\setcounter{page}{1}

\section{Introduction}

Online matching addresses the problem of dynamic resource allocation under uncertainty. In this model, a set of offline resources is known in advance, while requests arrive sequentially and must be irrevocably matched to maximize a cumulative objective. The crux lies in the online nature of the arrivals: the decision maker must balance the immediate benefit of a match against the opportunity cost of consuming a resource that may be more valuable in the future. Online matching has been a mainstay in domains such as Internet advertising~\citep{Huang24:Online,Mehta07:Adwords}. Its role is even more consequential in organ allocation, where the resources are scarce and life-saving.

Heart transplantation remains the therapy of choice for patients suffering from end-stage heart failure. However, the scarcity of donor organs has created a wide chasm between supply and demand~\citep{Cameli22:Donor}. This disparity necessitates an allocation mechanism that ensures efficient distribution of donor organs. Allocating donor hearts is inherently \emph{online}. Organs arrive sequentially, and decisions must be made irrevocably in a dynamic fashion. This immediacy is dictated by biological constraints: donor hearts have a limited viability window known as ischemic time, which precludes delaying transplantation. As a result, when a heart becomes available, the policy must immediately determine the recipient without the benefit of knowing future donors. 
%This places heart transplant allocation in the purview of online matching.

The US heart transplant allocation policy underwent a significant revision in  2018, transitioning from a 3-tier to a 6-tier system to better stratify candidate urgency~\citep{Kilic21:Evolving}. However, a critical limitation remains: the framework fails to integrate more granular estimates of pretransplant mortality risk and post-transplant survival prospects~\citep{Shore20:Changes,Zhang24:Development}. Instead, the policy relies on rigid, hierarchical priority statuses---handcrafted by medical committees---serving as coarse proxies of medical severity, often conflating clinically heterogeneous patients within the same tier. Furthermore, by focusing on waitlist severity, the system risks allocating scarce organs to candidates with low post-transplant survival~\citep{Cogswell20:Early}. These deficiencies have prompted calls for machine learning and other data-driven approaches to improve efficiency. One recent example is the framework of \emph{continuous distribution}, which has been deployed for lung allocation~\citep{Papalexopoulos24:Reshaping}.

In high-stakes domains, such as organ allocation, the deployment of policies demands rigorous, provable guarantees. Relying on heuristics can lead to suboptimal policies and introduces the risk of unforeseen failure modes. However, a considerable gap exists in the literature. On the one hand, the theoretical community studying online matching problems focuses on an overly pessimistic worst-case analyses. These works rely either on adversarial arrivals or general distributions that do not reflect the particular structure of medical data (\Cref{sec:related-matching}), resulting in overly conservative policies. On the other hand, applied policy optimization focuses on simulation-driven approaches that, while data-informed, inherently lack provable guarantees.

\begin{figure}
    \centering
    \scalebox{0.6}{\begin{tikzpicture}[
    font=\sffamily,
    >=Stealth,
    patient/.style={circle, fill=patientblue, inner sep=0pt, minimum size=0.15cm},
    % Donor is just a coordinate holder
    donor/.style={inner sep=0pt, minimum size=0pt},
    cluster/.style={circle, draw=patientblue!50, fill=clustergray, thick, minimum size=2.5cm, align=center},
    donortype/.style={rounded corners, draw=donorred!80, fill=white, thick, minimum width=2cm, minimum height=1cm, align=center},
    edge_weight/.style={midway, fill=white, inner sep=1pt, text=edgegreen, font=\footnotesize\bfseries}
]

% --- 1. The Timeline (Online Arrivals) ---
\draw[->, thick, gray] (-4, 4) -- (-4, -4.5) node[below] {Time};

% --- DONOR HEARTS (Complex Pic) ---
\node[
  donor,
  label={left:$t_1$}
] (d1) at (-4, 3)
  {\tikz{\pic[scale=0.06]{complexheart};}};

\node[
  donor,
  label={left:$t_2$}
] (d2) at (-4, 1.5)
  {\tikz{\pic[scale=0.06]{complexheart};}};

\node[
  donor,
  label={left:$t_3$}
] (d3) at (-4, -1)
  {\tikz{\pic[scale=0.06]{complexheart};}};

\node[
  donor,
  label={left:$t_4$}
] (d4) at (-4, -3)
  {\tikz{\pic[scale=0.06]{complexheart};}};

\node[align=center, font=\bfseries\small] at (-4, 4.5) {Online donor\\arrivals};

% --- 2. The Abstraction Layer (Donor Types) ---
\node[donortype] (typeA) at (-1, 2) {\scriptsize{Donor type A}};
\node[donortype] (typeB) at (-1, -2) {\scriptsize{Donor type B}};

% Mapping actual donors to types
\draw[dashed, gray] (d1) -- (typeA);
\draw[dashed, gray] (d2) -- (typeB);
\draw[dashed, gray] (d3) -- (typeA);
\draw[dashed, gray] (d4) -- (typeB);

% --- 3. The Patient Clusters (Offline Side) ---

% Cluster 1
\node[cluster] (c1) at (4, 2.5) {};
\node[above=0.1cm of c1, text=patientblue, font=\bfseries] {Patient cluster 1};

% Patients inside Cluster 1 (Dots)
\foreach \x/\y in {3.7/3.2, 4.3/2.5, 4.0/1.8} {
    \node[patient] at (\x,\y) {};
}

% Cluster 2
\node[cluster] (c2) at (4, -2.5) {};
\node[above=0.1cm of c2, text=patientblue, font=\bfseries] {Patient cluster 2};

% Patients inside Cluster 2
\foreach \x/\y in {3.8/-2.2, 4.5/-2.7, 3.6/-2.8} {
    \node[patient] at (\x,\y) {};
}

% --- 4. The Matching Edges ---

% Edges from Donor Types to Clusters
\draw[->, thick, edgegreen] (typeA.east) -- (c1.west) node[edge_weight, above=2pt] {High benefit};
\draw[->, thick, gray!50] (typeA.east) -- (c2.west);

\draw[->, thick, edgegreen] (typeB.east) -- (c2.west) node[edge_weight, below=2pt] {};
\draw[->, thick, gray!50] (typeB.east) -- (c1.west) node[midway, fill=white, inner sep=1pt, text=gray, font=\footnotesize] {Incompatible};

\end{tikzpicture}}
    \hspace{1cm}
    \includegraphics[scale=0.4]{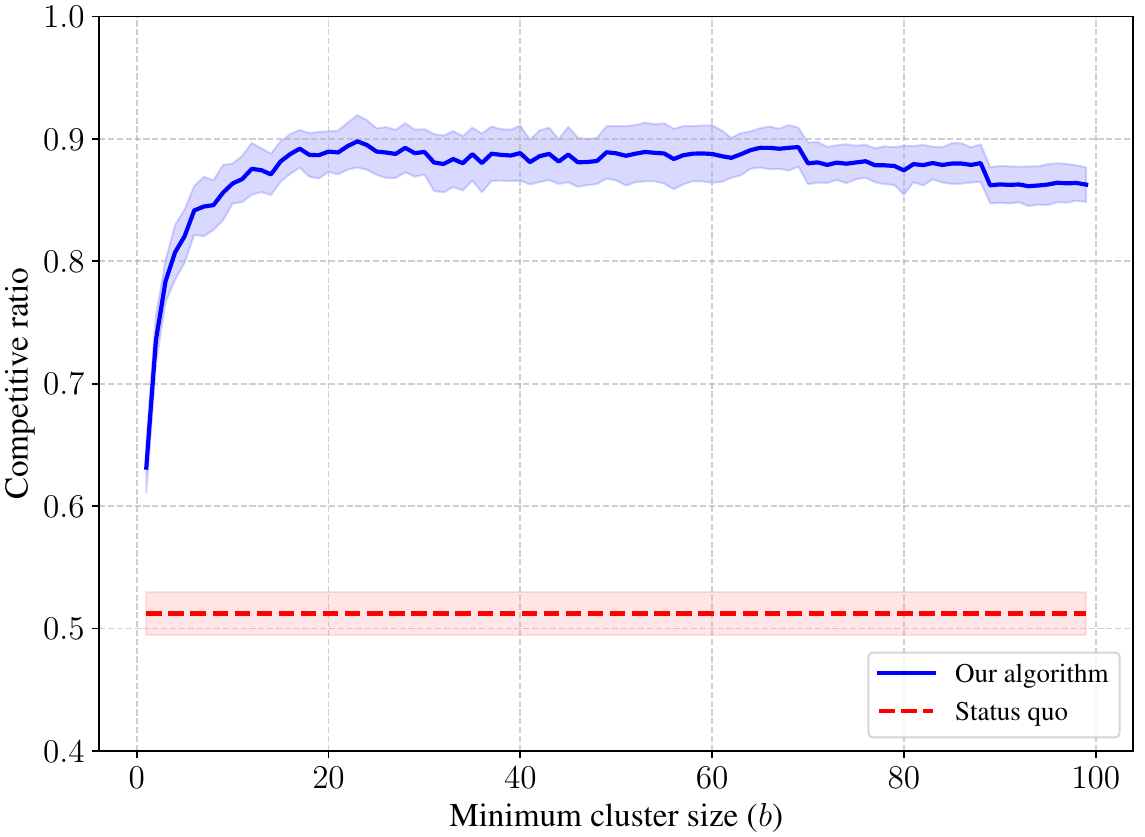}
    \caption{Left: Schematic illustration of our coarsening-based approach. Right: Competitive ratio of our algorithm compared against the post-2018 US \emph{status quo} allocation for different budgets $b$ evaluated in simulation from January to June 2019. }
    \label{fig:clusters}
\end{figure}
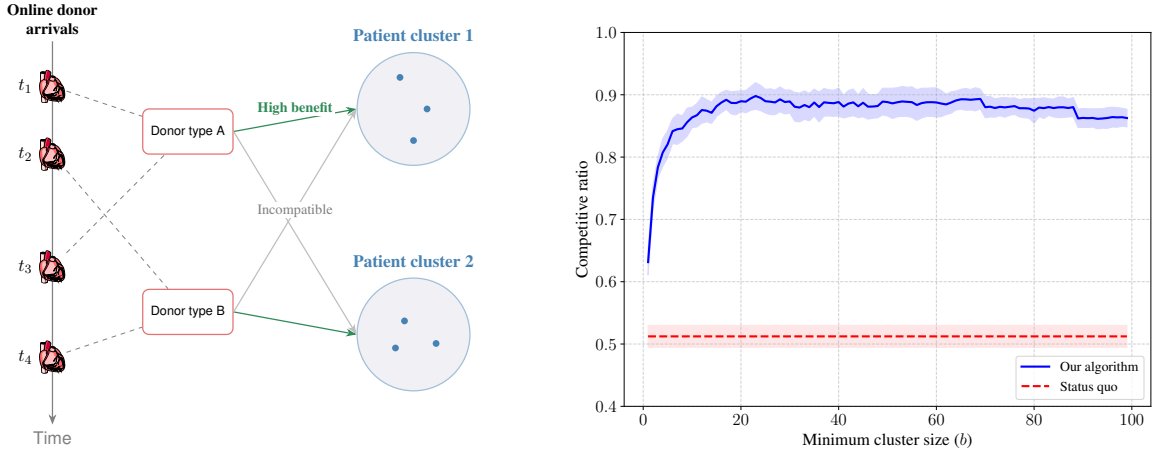

\subsection{Our contributions}

In this paper, we take an important step toward bridging this gap by developing theoretically grounded policies with strong practical performance. Our approach relies on a general online matching framework that we introduce based on \emph{coarsening}, whereby offline nodes are carefully aggregated into capacitated clusters (\Cref{fig:clusters}). While coarsening typically implies a loss of fidelity, we show that, to the contrary, leveraging such clusters can yield near-optimal theoretical guarantees. 

We apply our general methodology to heart transplant allocation. We identify a structural condition within the \textit{United Network for Organ Sharing (UNOS)} registry data, which comprises decades-long historical records. Specifically, by analyzing the (bipartite) graph weighted by expected life years gained, we reveal the presence of large clusters of patients who share similar edge weights. This enables us to circumvent standard worst-case lower bounds in the online matching literature, and renders our coarsening-based approach highly effective.

We then leverage this structural observation to derive a provably near-optimal policy. The goal is to optimize the \emph{competitive ratio}, which measures how well we perform \emph{vis-\`a-vis} the \emph{omniscient} policy which allocates with full foresight; that is, the competitive ratio captures the price of uncertainty caused by not knowing the future. We prove that when patient clusters are large and exhibit small error---a condition we corroborate in the UNOS data---the competitive ratio of our online algorithm approaches 1, that is, as good as matching with perfect foresight. In simulations on real UNOS data, we find that this theoretical result translates to strong practical performance: our policy delivers a competitive ratio of 0.91 (\Cref{fig:clusters}). This is drastically higher than non-capacitated stochastic matching's 0.63 and the US \emph{status quo} policy's 0.51. 

From a technical standpoint, our results are driven by a connection we make to the \emph{$b$-matching problem} from the literature on Internet advertising. In that context, $b$ serves as the \emph{budget} of the advertiser, which we interpret to be the size of the patient cluster in our context. This correspondence allows us to harness a rich repertoire of algorithmic techniques for organ allocation, adapting them to accommodate the specific constraints present in our problem. Our framework theoretically grounds heart transplant allocation in a decades-long line of work in online algorithms, and we anticipate that this connection will be fruitful in future research as well.

%We summarize our main contributions below.
%\begin{enumerate}[noitemsep]
%    \item We introduce a general coarsening-based framework for online stochastic matching and establish a near-optimal competitive ratio under cluster structure.
%    \item We apply this framework to heart transplant allocation, validating that real-world data contains large patient clusters sharing similar transplant benefits. We demonstrate through realistic simulations that our policy delivers a near-optimal competitive ratio.
%\end{enumerate}

\iffalse
\begin{figure}[t]
    \centering
    \ifthenelse{\boolean{isSingleColumn}}
    {
        % Use 0.5 column width
        \includegraphics[width=0.5\linewidth]{figures/competitive_ratio_summary.pdf}
    }
    {
        % Use full column width
        \includegraphics[width=\linewidth]{figures/competitive_ratio_summary.pdf}
    }
    \caption{Competitive ratio of our algorithm compared against the post-2018 US \emph{status quo} allocation for different budgets $b$ evaluated in simulation from January to June 2019.}
    \label{fig:competitive_ratio-summary}
\end{figure}
\fi

\section{Related work}

\paragraph{Organ allocation} Much of the work on data-driven organ allocation restricts optimization to simple policy classes. For instance, in the context of kidney transplantation, \citet{Bertsimas2013fairness} consider policies based on a point system that ranks patients according to fixed priority criteria. Similarly, the continuous distribution framework of~\citet{Papalexopoulos24:Reshaping} relies on a composite score, reducing the problem to optimizing over a small, fixed set of weights. The \emph{FutureMatch} approach of~\citet{Dickerson2015:Futurematch}, originally developed for kidney exchange and recently adapted to heart transplantation allocation~\citep{Anagnostides25:Policy}, is similarly based on a heuristic addition referred to as ``potentials''~\citep{Dickerson2012:Dynamic}. In the realm of survival estimation, the \emph{OrganITE} framework prioritizes matches based on the estimated individual treatment effects and organ scarcity~\citep{Berrevoets20:OrganITE}.  Its successor, \emph{OrganSync}~\citep{Berrevoets21:Learning}, moves closer to our approach by employing multiple priority queues---one for each derived organ type.

While these approaches have the advantage of interpretability and computational tractability, they inherently restrict the solution space. We argue that such low-dimensional parameterizations and heuristic approaches are not expressive enough to capture the full complexity of the online matching problem.

Patient clustering has also previously been employed in organ allocation frameworks such as \emph{OrganSync} \citep{Berrevoets21:Learning}. Their motivation for using clustering revolves around statistical estimation---to derive more robust counterfactual estimates of post-transplant survival---while we use them for algorithmic purposes. Although the underlying policies are considerably different, our paper is complementary and provides a missing theoretical underpinning for such clustering-based methods. Conversely, the framework of~\citet{Berrevoets21:Learning} further supports our approach: coarsening not only yields the algorithmic improvements we establish, but also mitigates estimation bias.

\paragraph{Online matching} There has been considerable work on online matching problems, primarily motivated by Internet advertising applications~\citep{Huang24:Online,Mehta10:Online}. Online matching is modeled on a bipartite graph $G = (U, V, E)$, where $U$ is the set of offline nodes and $V$ is the set of nodes arriving online. In the basic version, each offline node $u \in U $ can be matched to at most one online node. Of particular relevance to us is the \emph{$b$-matching} generalization of the problem, in which a node $u \in U$ can be matched to up to $b \in \N$ nodes in $V$~\citep{Kalyanasundaram00:Optimal,Brubach16:New,Alaei12:Online,Devanur09:Adwords}. In the context of Internet advertising, $U$ represents the set of \emph{advertisers}, while $V$ is the stream of \emph{impressions}. The edge set $E$ captures the compatibility, where an edge $(u, v)$ exists if advertiser $u$ is interested in bidding on impression $v$. 

In practical settings, to enhance computational speed of optimizing the dispatch plan (using a linear or integer program), automated abstraction methods are often employed to bundle inventory or segments, effectively grouping distinct instances into larger $b$-matching constraints~\citep{Walsh2010:Automated,Peng2016:Scalable}. 
%
%Our approach in~\Cref{sec:theory} leverages the framework of~\citet{Brubach16:New} which also accommodates stochastic rewards. 
%
However, to our knowledge, such ideas have not been employed for the purpose of hedging against future uncertainty, nor have they been applied to  organ allocation for any reason. We provide further background on online matching in~\Cref{sec:related-matching}.

\section{Coarsening-based online stochastic matching}
\label{sec:theory}

In this section, we present our theoretical results concerning our coarsening-based approach for online stochastic matching. To begin with, we recall the basic approach of~\citet{Brubach16:New} for the special case of $b$-matching with stochastic rewards.

\subsection{Linear programming formulation of $b$-matching}

In the $b$-matching problem, it is assumed that each offline node $u \in U$ possesses a uniform integral capacity $b$, meaning that it can be matched at most $b$ times. The goal is to maximize the \emph{competitive ratio} of the algorithm $\alg$, defined as
\[
c = \frac{\E[\alg]}{\E[\opt]}.
\]
To obtain an algorithm with a near-optimal competitive ratio, \citet{Brubach16:New} consider the following linear program, which is to be solved \emph{once} offline.
\begin{align}
\max \quad & \sum_{e \in E} w_e f_e p_e \label{eq:obj} \\
\text{s.t.} \quad & \sum_{e \in \partial(u)} f_e p_e \leq b \quad \forall u \in U, \label{eq:capacity} \\
& \sum_{e \in \partial(v)} f_e \leq r_v \quad \forall v \in V. \label{eq:arrival}
\end{align}

Above, $r_v$ is the arrival rate of node $v$; if $T \in \N$ is the total number of rounds, $\sum_{v \in V} r_v = T$. The arrival rates are allowed to be fractional. $\partial (u)$ denotes the set of edges incident to $u$. For each edge $e \in E$, $p_e$ denotes the probability that the edge $e = (u, v)$ will be present when $v$ is assigned to $u$. The stochastic process underpinning $p_e$ is assumed to be independent of the stochastic arrival of each $v$. In the context of organ allocation, $p_e$ has a clear interpretation: after a match is made, there is a probability that it will not proceed to transplantation. While our theoretical analysis incorporates the edge failure $p_e$, our simulations assume that the identified matches are successful.

\Cref{alg:bmatching} presents the basic approach of~\citet{Brubach16:New} for the $b$-matching problem.

\begin{algorithm}[t]
\caption{Stochastic matching for $b$-matching ($\mathsf{SM}_b$)}
\label{alg:bmatching}
\begin{algorithmic}[1]
\INPUT weighted graph $G=(U, V, E)$, capacity $b$, arrival rates $\{r_{v} \}_{v \in V}$.
\STATE \textbf{Offline Phase}:
\STATE Compute an optimal solution $\{f_e\}_{e \in E}$ to the linear program \eqref{eq:obj}--\eqref{eq:arrival}.
\STATE \textbf{Online Phase}:
% CHANGE: Make this a loop to trigger indentation
\FOR{each arriving node $v \in V$} 
    \FOR{each neighbor $u$}
        \STATE Independently (attempt to) match $(u, v)$ with probability $\frac{f_{u,v}}{r_v}$.
    \ENDFOR
\ENDFOR
\end{algorithmic}
\end{algorithm}

Notably, ~\Cref{alg:bmatching} is \emph{non-adaptive}. We also point out that it can match a node $u$ even after its budget has been exceeded, in which case the match becomes indeterminate. The key point, however, is that no matter how those indeterminacies are resolved, \Cref{alg:bmatching} attains a near-optimal competitive ratio:

\begin{restatable}[\citealp{Brubach16:New}]{theorem}{bmatching}
\label{thm:bmatching}
For edge-weighted online stochastic $b$-matching with arbitrary arrival rates and stochastic rewards, the online algorithm $\mathsf{SM}_b$ (\Cref{alg:bmatching}) achieves a competitive ratio of at least $1 - b^{-1/2+\epsilon} - e^{-b^{2 \epsilon }/3}$ for any $\epsilon > 0$.
\end{restatable}

More concretely, a budget of size $b = 100$ delivers in expectation at least $\approx 90 \%$ of the optimal in hindsight. For completeness, we provide the proof of~\Cref{thm:bmatching} in~\Cref{sec:proofs}, together with all other proofs of claims in our paper. In what follows, we use the shorthand notation $\alpha(b) \defeq \sup_{\epsilon > 0} 1 - b^{-1/2+\epsilon} - e^{-b^{2 \epsilon }/3}$.

\subsection{Coarsening-based matching}

We now build on the previous result to develop coarsening-based algorithms for online matching. We assume that each node $u \in U$ belongs to a \emph{cluster} consisting of $b$ individual nodes (for example, patients), which we denote by $\{ u_1, \dots, u_b \}$. (Without any loss, we can assume that $|U|$ is divisible by $b$.) Unlike the standard $b$-matching setting where all $b$ copies are identical, here we assume that the nodes within a cluster are heterogeneous but similar. For the sake of the analysis, we assume for now that a $b$-clustering $\mathcal{U}$ is given, but ultimately our algorithm proceeds by constructing these clusters (\Cref{alg:clustered_matching}).

To begin with, we posit a \emph{uniform} bound on the intra-cluster heterogeneity. Specifically, we make the following assumption.

\begin{assumption}[Bounded cluster variance]
\label{ass:cluster_relative}
There exists a $b$-clustering $\mathcal{U}$ partitioning the node set $U$ such that for every cluster $u \in \mathcal{U}$, node $v \in V$, and any node $u_k \in \{u_1, \dots, u_b\}$ within the cluster, there exists $\bar{w}_{u, v}$ such that the true weights satisfy
\begin{equation}
    \label{eq:rel-error}
    (1 - \delta(b)) \bar{w}_{u,v} \leq w_{u_k,v} \leq (1 + \delta(b)) \bar{w}_{u,v}
\end{equation}
for some $\delta(b) \in [0,1)$.
\end{assumption}
Specifically, \eqref{eq:rel-error} imposes a \emph{relative} error bound. We will refer to $\bar{w}_{u,v}$ as the \emph{representative weight} for matching node $v$ to cluster $u$. To simplify the exposition, we will assume throughout that the edge failures satisfy $p_{u_k, v} = p_{u, v}$; our analysis can accommodate some error here as well. 

\Cref{ass:cluster_relative} makes explicit the fact that the intra-cluster error is a function of the cluster's size. There is a clear tradeoff that we will return to: increasing the size of the cluster $b$ can improve the competitive ratio in the representative instance when invoking~\Cref{thm:bmatching}, but at the cost of increasing $\delta(b)$.

To extend~\Cref{thm:bmatching} under~\Cref{ass:cluster_relative}, the basic idea is to solve the linear program in \eqref{eq:obj}--\eqref{eq:arrival} using the \emph{representative weights}. In~\Cref{sec:proofs}, we show that this approach delivers a competitive ratio of at least $\alpha(b) (1 - 2 \delta(b))$. 

Now, in our setting, $b$ is not a fixed constraint, but a design parameter to be optimized. Our analysis reveals a way to tune it: find the size of the cluster from a set of parameters $\mathcal{B}$ that maximizes $\alpha(b) (1 - 2 \delta(b))$ with respect to $b$. This whole process is summarized in~\Cref{alg:clustered_matching}.

\begin{restatable}{theorem}{clusteredbmatching}
\label{thm:clustered_bmatching}
Consider a weighted online stochastic matching problem satisfying \Cref{ass:cluster_relative}. If~$\mathsf{ALG}$ is set to~\Cref{alg:clustered_matching},
\begin{equation*}
\frac{\E[\mathsf{ALG}]}{\E[\opt]} \geq \max_{b \in \mathcal{B}} \alpha(b) ( 1 - 2 \delta(b)).
\end{equation*}
\end{restatable}

\begin{algorithm}[t]
\caption{Clustered stochastic matching ($\mathsf{CSM}$)}
\label{alg:clustered_matching}
\begin{algorithmic}[1]
\INPUT weighted graph $G = (U, V, w)$, rates $\{r_{v} \}_{v \in V}$.
\STATE \textbf{Offline phase}:
\STATE Initialize $c^* \leftarrow 0$, configuration $(\mathcal{U}^*, b^*) \leftarrow (\emptyset, 1)$.
\FOR{each size $b \in \mathcal{B}$}
    \STATE Find a partition $\mathcal{U}_b$ of the node set $U$ into size-$b$ clusters minimizing $\delta(b)$ (per \Cref{ass:cluster_relative}).
    \STATE Calculate a lower bound on the competitive ratio $c \leftarrow \alpha(b) (1 - 2 \delta(b))$.
    \IF{$c > c^*$}
        \STATE $c^* \leftarrow c, \mathcal{U}^* \leftarrow \mathcal{U}_b, b^* \leftarrow b$.
    \ENDIF
\ENDFOR
\STATE Compute an optimal solution $\{f_e \}_{e \in E}$ of the LP \eqref{eq:obj}--\eqref{eq:arrival} w.r.t. $(\mathcal{U}^*, V, \bar{w})$ with capacity $b^*$.
\STATE \textbf{Online phase}:
\FOR{each arriving node $v \in V$}
    \FOR{each cluster $u \in \mathcal{U}^*$}
        \STATE Independently (attempt to) match $v$ to any free $u_k \in u$ with probability $\frac{f_{u,v}}{r_v}$.
    \ENDFOR
\ENDFOR
\end{algorithmic}
\end{algorithm}

\Cref{alg:clustered_matching} abstracts away the cluster construction; we address the specific implementation of this step in~\Cref{sec:clustering}.

\paragraph{Heterogeneous error analysis} While our previous analysis assumes a uniform bound across all clusters, this can be overly conservative in applications where intra-cluster error varies significantly, such as heart transplant allocation (\Cref{sec:app}). We address this in~\Cref{thm:clustered_avg} (\Cref{sec:proofs}), where we relax this assumption to allow for cluster-dependent errors. Our analysis depends on (\textit{ex post}) \emph{average weighted error}, showing that high errors are tolerable provided they occur in low-value clusters or clusters that are infrequently selected. A simplified bound along these lines is given later in~\Cref{thm:bad_clusters}.

\paragraph{Accounting for errors in the weights} Prior work in online matching typically assumes precise knowledge of edge weights. In organ allocation, an edge weight typically represents survival benefit, which is merely a statistical estimate. In heart transplantation, predicting graft survival is a notoriously challenging problem, with even state-of-the-art models exhibiting a high predictive error~\citep{Lee18:Deephit}. This highlights the need for robust algorithms that remain effective even under substantial estimation error.

In this context, we observe that our analysis can accommodate these estimation errors. Specifically, it suffices to work under the following relaxed assumption.
\begin{assumption}[Bounded estimation error]
\label{ass:imprecise_weights}
Let $\{w^*_e\}_{e \in E}$ denote the unknown, true edge weights. We assume that the given weights $\{w_e\}_{e \in E}$ are such that for all $e \in E$,
\[ (1 - \eta ) w_e \leq w^*_e \leq (1+\eta)w_e, \]
where $\eta \in [0, 1)$ represents the relative error margin.
\end{assumption}
In other words, while performance is measured in terms of the true edge weights $\{ w^*_e \}_{e \in E}$, we only have access to $\{ w_e \}_{e \in E}$. With this, we can extend~\Cref{thm:clustered_bmatching} as follows.

\begin{theorem}
    \label{thm:bounded-est}
Consider a weighted online stochastic matching problem with $\eta$-bounded estimation error (\Cref{ass:imprecise_weights}) satisfying \Cref{ass:cluster_relative} relative to the given weights. Let $\mathsf{ALG}$ be set to~\Cref{alg:clustered_matching}. Then
\begin{equation*}
\frac{\E[\mathsf{ALG}]}{\E[\opt]} \geq (1 - 2\eta) \max_{b \in \mathcal{B}} \alpha(b) ( 1 - 2 \delta(b)).
\end{equation*}
\end{theorem}

To conclude the section, we strengthen~\Cref{thm:clustered_bmatching} by assuming that a small fraction of clusters can have a large error.

\begin{assumption}[Most clusters have bounded error]
\label{ass:most_cluster_good}
There exists a $b$-clustering $\mathcal{U}$ that can be partitioned into $\Ugood$ and $\Ubad = \mathcal{U} \setminus \Ugood$ such that the following properties hold.
\begin{itemize}
    \item For every cluster $u \in \Ugood$, node $v \in V$, and node $u_k \in \{u_1, \dots, u_b\}$ within the cluster, there exists $\bar{w}_{u, v}$ such that
\begin{equation*}
    (1 - \delta(b)) \bar{w}_{u,v} \leq w_{u_k,v} \leq (1 + \delta(b)) \bar{w}_{u,v}
\end{equation*}
for some $\delta(b) \in [0, 1)$.
    \item There exists $\rho(b) \ll 1$ such that the weight of any matching on $\Ubad$ is at most $\rho(b) \opt$.
\end{itemize}
\end{assumption}
Under this assumption, we obtain the following bound on the competitive ratio.

\begin{restatable}{theorem}{badclusters}
\label{thm:bad_clusters}
    Consider an edge-weighted online stochastic matching problem satisfying \Cref{ass:most_cluster_good}. There exists an algorithm $\mathsf{ALG}$ such that
    \begin{equation*}
        \E[\alg] \geq \max_{b \in \mathcal{B}} \alpha(b) (1 - \rho(b)) ( 1 - 2 \delta(b)) \E[\opt].
    \end{equation*}
\end{restatable}

On top of~\Cref{ass:most_cluster_good},
we can also account for estimation error per~\Cref{ass:imprecise_weights}, incurring a multiplicative degradation of $1 - 2\eta$ in the competitive ratio (\Cref{thm:bounded-est-most}).

\begin{remark}
    The results presented in this section do not require the clusters to be of equal size. In the general case, $b$ in our bounds denotes the minimum cluster size.
\end{remark}
\section{Capacitated $b$-clustering}\label{sec:clustering}

\Cref{alg:clustered_matching} relies on a subroutine to partition the set of offline vertices $U$ into \textit{capacitated} clusters, meaning each group has size at least $b$. To maximize the competitive ratio derived in \Cref{thm:clustered_bmatching}, we desire a clustering algorithm to minimize the intra-cluster error. We stress that any capacitated clustering approach is compatible with \Cref{alg:clustered_matching}, provided it satisfies the minimum size constraint.

However, the specific choice of clustering algorithm is constrained by the scale of real-world heart transplant data. A typical bipartite graph representing a period of multiple months involves thousands of waiting patients and potential donors, resulting in millions of edges. Finding an optimal clustering that minimizes the cluster radius subject to minimum size constraints is an \NP-hard problem~\citep{Aloise2009:NP}. Exact formulations, such as MIP-based capacitated clustering~\citep{Mulvey1984:Solving}, scale poorly with the size of the graph.

To address the scale of the problem, we evaluate three heuristic approximations: \textit{constrained agglomerative clustering}~\citep{Ward1963:Hierarchical}, \textit{constrained $k$-means}~\citep{Bradley2000:Constrained}, and \textit{recursive bisection}~\citep{Karypis2000:Comparison}. More details of these algorithms are provided in \Cref{sec:appendix-clustering}. In fact, we show in \Cref{sec:results} that the specific choice of clustering algorithm has little impact on the effectiveness of our approach.

\subsection{Distance measures and representative weights}

Let $U$ denote the set of offline vertices (patients) in the bipartite graph. Associated with each vertex $u \in U$ is a utility vector representing the edge weights to the online vertices.

The clustering approaches utilize the standard Euclidean distance between the utility vectors. While our theoretical guarantees rely on the maximum relative error, directly optimizing this objective is computationally prohibitive at scale. Minimizing Euclidean distance is a tractable proxy, as it implicitly constrains the worst-case deviations in dense datasets. One could consider other measures such as the $\ell_\infty$ norm. However, there is a tradeoff between the two measures; using $\ell_\infty$ optimizes the worst-case at the expense of the average. 

Once the clustering $\mathcal{U}$ is obtained, we compute the representative weights required for the offline linear program in \Cref{alg:clustered_matching}. For each cluster $u \in \mathcal{U}$ comprising nodes $\{u_1, \dots, u_{b}\}$, where $b$ is the size of the cluster, and for each online vertex $v \in V$, the representative weight $\bar{w}_{u,v}$ is the mean of the constituent utility vectors:
\begin{equation*}
    \bar{w}_{u,v} = \frac{1}{b} \sum_{i=1}^{b} w_{u_i, v}.
\end{equation*}
The representative vector $\bar{w}_{u}$ is the expected utility obtained if vertex $v$ is routed to cluster $u$ and matched to a node chosen uniformly at random from within that cluster. Consequently, the representative weight provides an unbiased estimator of the realized utility in the online phase.
\section{Application to heart transplant allocations}
\label{sec:app}

We model the patient-donor matching problem for heart transplant allocation with a bipartite graph. The set of offline nodes $U$ corresponds to the set of patients on the waitlist. In the $b$-matching model, each node $u \in U$ is a representative patient type (\textit{i.e.}, cluster). Each online node $v \in V$ corresponds to a donor arrival. We define an edge $e$ between a patient and donor node if the donor is compatible with the patient (\emph{e.g.}, based on clinical factors/geography). The edge weight $w_e$ quantifies the predicted utility of the transplant. We adopt the utilitarian metric of \textit{population life years gained}, which measures the expected increase in survival time for a patient receiving a donor organ versus remaining on the waitlist. This is a standard metric used in organ allocation research  (\emph{e.g.},~\citealp{Berrevoets20:OrganITE,Berrevoets21:Learning}). There are other natural objectives to consider---such as fairness with respect to different demographic groups---and our framework can be adapted accordingly. For example, the edge weights can be adjusted to prioritize marginalized populations, as was done in kidney exchange~\citep{Dickerson2015:Futurematch}.

To estimate these weights, we train a Cox proportional hazards model on historical outcomes dating back to 1987~\citep{Cox1972:Regression}. More sophisticated models have been developed~\citep{Lee18:Deephit}, but their concordance is only marginally better than Cox regression.

\subsection{Waitlist population stability}
\label{sec:waitlist_stability}

We begin by motivating the use of online stochastic $b$-matching by illustrating the stability of the distribution of edge weights within the waitlist population. While the specific composition of the transplant waitlist changes as patients are added or removed, the underlying distribution of potential match qualities often remains consistent. To quantify this, we track the \emph{population stability index} ($\PSI$), a statistical measure of change in a distribution of data between two datasets. A $\PSI$ value below $0.1$ indicates that the distributions are nearly identical, while values between $0.1$ and $0.25$ suggest a minor shift, and values exceeding $0.25$ signify a major change in population characteristics. The index is calculated by binning the data into $N$ buckets and comparing the frequencies of the expected and actual distributions:
$$ \mathsf{PSI} = \sum_{i=1}^N \left( P_i - Q_i\right) \ln \frac{P_i}{Q_i}$$
where $P_i$ and $Q_i$ represent the fraction of observations falling into the $i$-th bin for the actual and expected datasets, respectively.

\ifthenelse{\boolean{isSingleColumn}}
{
\begin{figure}[t!]
    \centering
    \begin{subfigure}[b]{0.24\linewidth}
        \centering
        \includegraphics[width=\linewidth]{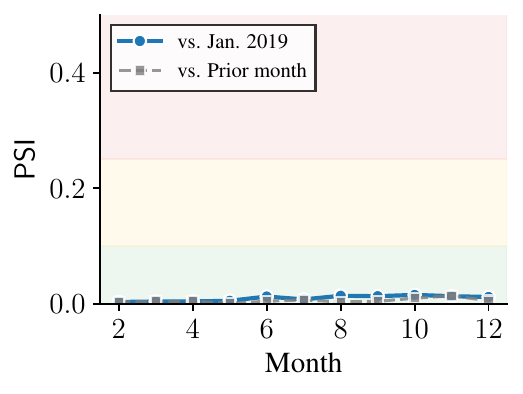}
        \caption{Blood type O.}
        \label{fig:psi_o}
    \end{subfigure}
    \hfill
    \begin{subfigure}[b]{0.24\linewidth}
        \centering
        \includegraphics[width=\linewidth]{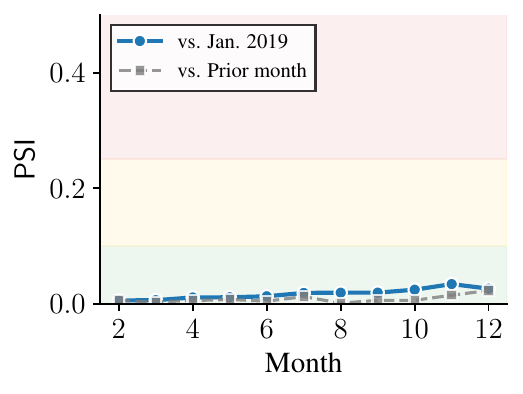}
        \caption{Blood type A.}
        \label{fig:psi_a}
    \end{subfigure}
    \hfill
    \begin{subfigure}[b]{0.24\linewidth}
        \centering
        \includegraphics[width=\linewidth]{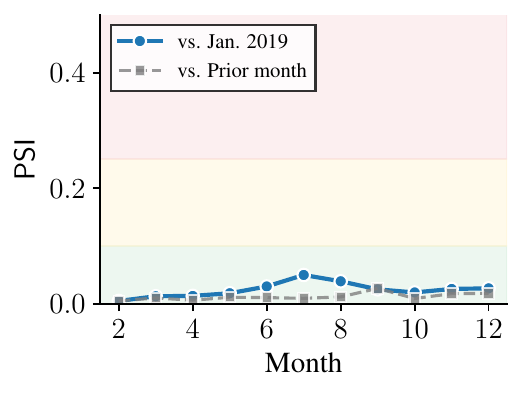}
        \caption{Blood type B.}
        \label{fig:psi_b}
    \end{subfigure}
    \hfill
    \begin{subfigure}[b]{0.24\linewidth}
        \centering
        \includegraphics[width=\linewidth]        {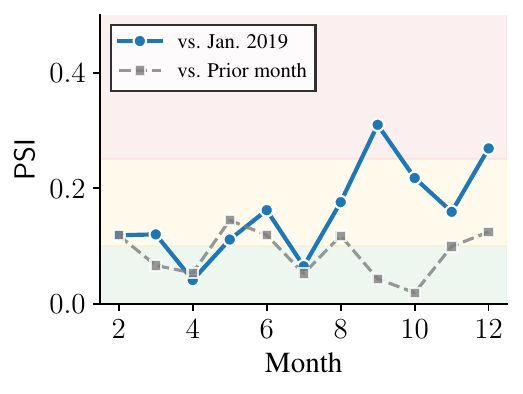}
        \caption{Blood type AB.}
        \label{fig:psi_ab}
    \end{subfigure}    
    \caption{Population stability index ($\PSI$) of the distribution of waitlist edge weights across blood types over 2019. We report the $\PSI$ measured against the baseline month of January 2019 and each prior month. A $\PSI \le 0.1$ constitutes a stable population while a $\PSI\ge 0.25$ denotes major distributional shift. We take $N=10$ bins.}
    \label{fig:psi}
\end{figure}
}
{    
\begin{figure}[t!]
    \centering
    \begin{subfigure}[b]{0.49\linewidth}
        \centering
        \includegraphics[width=\linewidth]{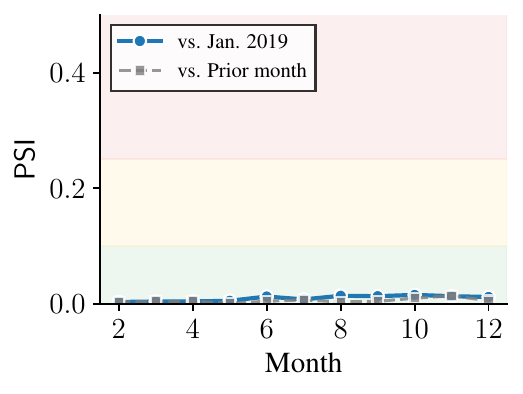}
        \caption{Blood type O.}
        \label{fig:psi_o}
    \end{subfigure}
    \hfill
    \begin{subfigure}[b]{0.49\linewidth}
        \centering
        \includegraphics[width=\linewidth]{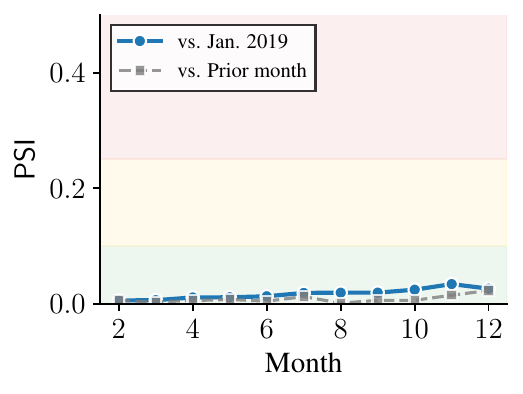}
        \caption{Blood type A.}
        \label{fig:psi_a}
    \end{subfigure}
    \vspace{0.5em} % Vertical spacing between rows
    \begin{subfigure}[b]{0.49\linewidth}
        \centering
        \includegraphics[width=\linewidth]{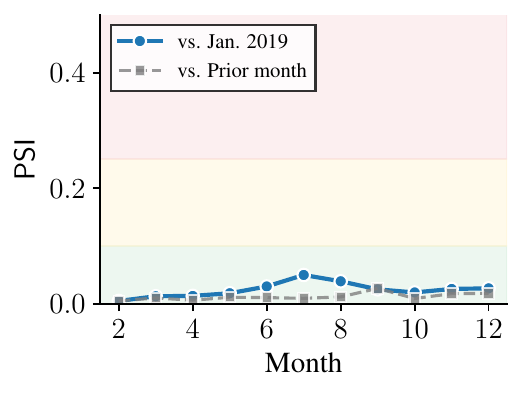}
        \caption{Blood type B.}
        \label{fig:psi_b}
    \end{subfigure}
    \hfill
    \begin{subfigure}[b]{0.49\linewidth}
        \centering
        \includegraphics[width=\linewidth]        {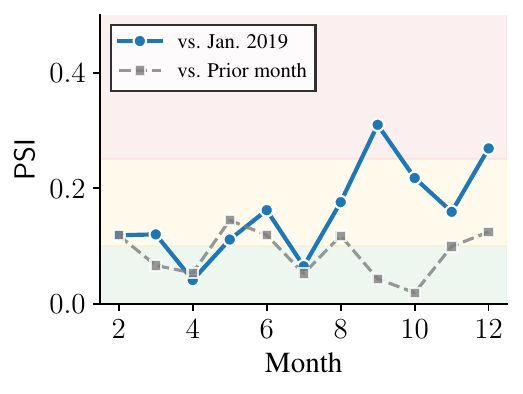}
        \caption{Blood type AB.}
        \label{fig:psi_ab}
    \end{subfigure}    
    \caption{Population stability index ($\PSI$) of the distribution of waitlist edge weights across blood types over 2019. We report the $\PSI$ measured against the baseline month of January 2019 and each prior month. A $\PSI \le 0.1$ constitutes a stable population while a $\PSI\ge 0.25$ denotes major distributional shift. We take $N=10$ bins.}
    \label{fig:psi}
\end{figure}
}

Figure \ref{fig:psi} illustrates the $\PSI$ for edge weights for adult heart transplantation candidates in the US throughout 2019. We separate the patient population by blood type as ABO compatibility is the primary determinant of match viability. We find that for the most prevalent blood types (O, A, and B), the distribution of edge weights is remarkably stable over time, consistently remaining below the $0.1$ threshold. The observed instability in the AB blood type is likely an artifact of its rarity, where small fluctuations in low-frequency counts lead to disproportionate inflation of the $\PSI$.

This distributional stability provides a strong foundation for online stochastic $b$-matching, as it ensures that the offline linear program remains representative of the waitlist landscape even as individual patients change. 

\subsection{Discretizing donor types}

We approximate the continuous space of donor attributes with a finite set of discrete types to facilitate online stochastic matching. Although the theoretical bipartite matching model supports an unbounded number of online nodes, there is a practical tradeoff between distributional granularity and computational tractability. We therefore discretize donor profiles into a manageable set of representative nodes to limit the complexity of the offline linear program and prevent excessive granularity in the arrival rates.

We generate discretized donor types by clustering the UNOS registry of deceased donors dating back to 1987. We partition the historical donor set into four disjoint subsets corresponding to blood types O, A, B, and AB. Blood type compatibility determines edge existence in the bipartite graph. Therefore, partitioning the data ensures that blood type remains consistent across a cluster. Within each partition, we apply $k$-means clustering on a feature vector of $34$ clinical attributes. To determine the optimal number of representative types, we examine the tradeoff between model complexity and information loss (see \Cref{tab:donor_clustering_results}). As the number of clusters grows to $k=1,000$ we obtain near perfect reconstruction. With $k=250$ clusters, we capture approximately $88$--$93\%$ of the variance regardless of the sub-population size. The results indicate that the high-dimensional donor space can be efficiently compressed with minimal information loss. We fix $k=250$ resulting in $1,000$ total discrete donor types across the four blood types to balance accuracy with scalability. For each donor type, we compute the arrival probability as the fraction of historical donors assigned to the representative cluster and define the attributes using the cluster centroid.

\subsection{Near-optimal heart transplant allocation}\label{sec:results}

We now demonstrate that using a variation of~\Cref{alg:clustered_matching} in combination with patient clustering yields near-optimal heart transplant allocations. We construct a waitlist using the UNOS registry from January through June 2019, comprising 3,113 patients. We partition the patients into groups of capacity at least $b$ using the methods described in \Cref{sec:clustering}. We simulate distinct month-long horizons where donor arrivals follow a Poisson process parameterized by the historical arrival rates. While we assume that all offers proceed to transplant, our framework naturally accommodates stochastic acceptance by interpreting edge weights as expected values as described in~\Cref{sec:theory}.

We adapt~\Cref{alg:clustered_matching} to show both theoretical and practical outcomes. We initialize the capacity of each representative node to the actual cluster size rather than the uniform lower-bound $b$. This prevents artificial bottlenecks in large, low-priority groups and ensures viable matches are not unnecessarily wasted. We also use a conservative dispatch protocol; if there are no available patients to match according to the offline strategy, the organ is discarded. One might hypothesize that this dispatch unfairly penalizes the $b=1$ baseline. However, we show in \Cref{sec:appendix-resampling} that our framework retains its dominance even when provisional re-routing is permitted to avoid discarding. Within a chosen cluster, our algorithm selects an available patient uniformly at random.  We employ this selection policy to show that performance gains are driven by the coarsening approach, rather than any sophisticated patient selection. We show in \Cref{sec:appendix-greedy} that a greedy intra-cluster selection policy yields even higher competitive ratios, surpassing greedy matching baselines in market-clearing scenarios. In reality, the patient selection may be driven by clinical factors such as location, urgency, or waiting time. 

We begin by analyzing the error introduced by the clustering approaches. We report the average and maximum \textit{mean normalized absolute error (NMAE)} (reviewed in~\Cref{sec:appendix-hcr}) in \Cref{fig:cluster-error}. We examine NMAE because the standard relative error is sensitive to near-zero denominators. The relative error of the clustering grows significantly. However, these extreme errors occur in clusters with low transplant benefit; clusters for whom a specific donor type provides almost zero benefit. Since the offline linear program maximizes total utility, it naturally routes flow away from these suboptimal clusters. Consequently, the error in these low-utility clusters has negligible impact on the empirical performance. The NMAE better captures the error in good clusters where actual matches occur. While the average NMAE remains near 0, the maximum climbs to 0.5, indicating that while most clusters are compact, the worst-case grows with $b$. NMAE can also be used as a heuristic to adjust the capacity $b$ (\Cref{sec:appendix-clustering}). An alternative method is to simulate the stochastic process and record outcomes for different values of $b$.

\begin{figure}[t]
    \centering
    \ifthenelse{\boolean{isSingleColumn}}
    {
        % Use 0.5 column width
        \includegraphics[width=0.5\linewidth]{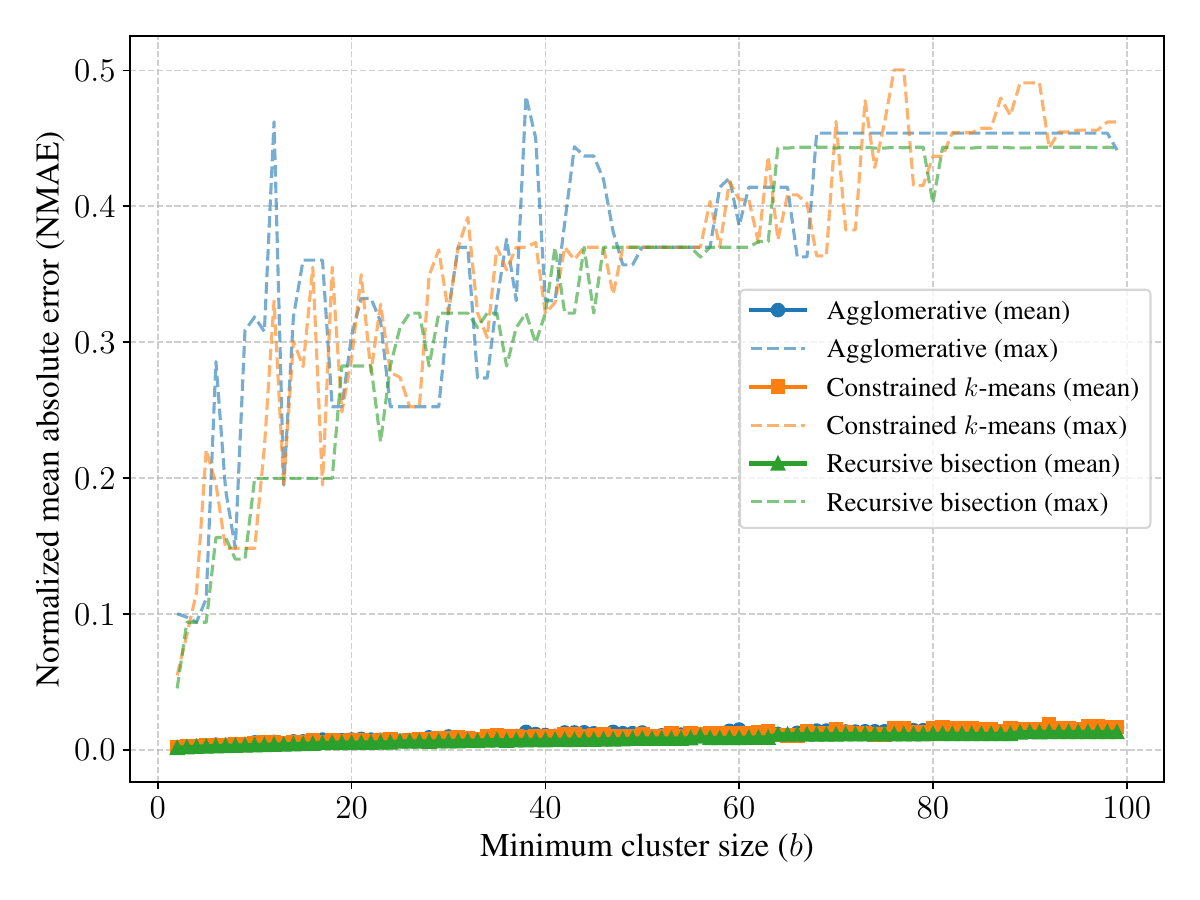}
    }
    {
        % Use full column width
        \includegraphics[width=0.99\linewidth]{figures/plot_nmae_analysis.pdf}
    }
    \caption{Normalized mean absolute error of clustering algorithms across cluster sizes.}
    \label{fig:cluster-error}
\end{figure}

\Cref{fig:competitive_ratio} displays the actual competitive ratio of our algorithm against the optimal offline hindsight solution evaluated in simulation. We also compare against the current \emph{status quo} policy in the US for allocating adult heart transplants (see~\Cref{appendix:statusquo} for more details). For the baseline case of $b=1$ (\textit{i.e.}, standard online stochastic matching), performance is consistent with the theoretical guarantee of $1-1/e$. Notably, as we increase the cluster size $b$, the competitive ratio improves significantly, peaking at $0.91$. For most values of $b$, the choice of clustering algorithm has little impact on the competitive ratio, demonstrating the robustness of the framework. Even for very coarse clustering, performance remains stable, only suffering a slight decline for large $b$. On the other hand, the \emph{status quo} policy achieves an average competitive ratio just above 0.5, significantly below our approach.

To contextualize the results, the optimal omniscient algorithm delivers an average utility of around 4,650 life years per month. This corresponds to approximately 300 to 400 successful transplants. Our coarsening-based approach delivers roughly 100 more transplants per month than the $b=1$ baseline, and 150 more than the \emph{status quo}. The substantial gain at modest $b$ shows that clustering acts as a buffer against uncertainty. By routing donors to a group rather than an individual, the algorithm smooths out the inefficiencies of stochastic arrivals without losing the precision needed for high-quality matching. 

Our approach achieves a dual advantage: it obtains more utility than individual routing while simultaneously reducing the run time by an order of magnitude (\Cref{sec:appendix-runtime}). The computational efficiency of our framework enables more adaptability to distributional shifts in either donor arrivals or waitlist composition. When inconsistencies are detected, we can dynamically replan a new dispatch. For example, replanning can be triggered when the $\PSI$ of the waitlist exceeds some threshold. 

\begin{figure}
    \centering
    \ifthenelse{\boolean{isSingleColumn}}
    {
        % Use 0.5 column width
        \includegraphics[width=0.5\linewidth]{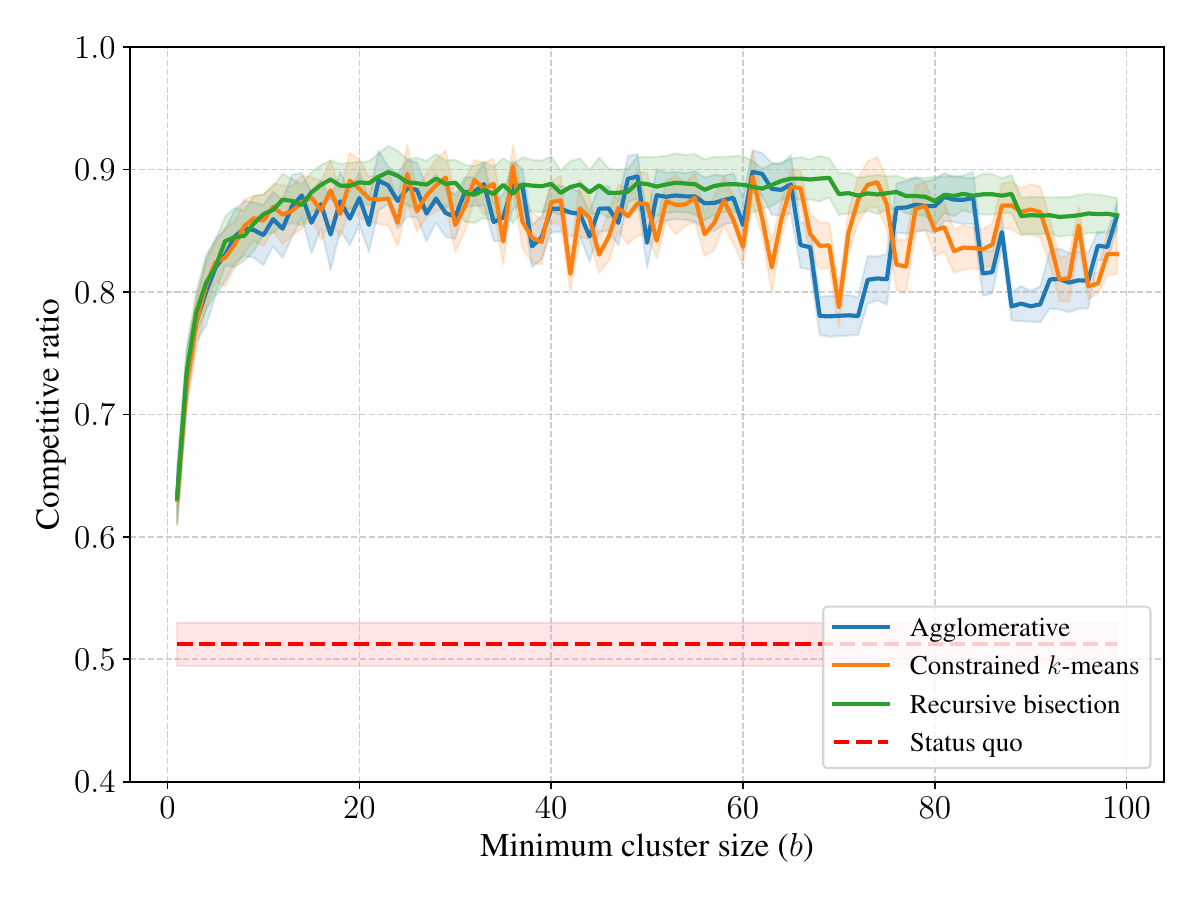}
    }
    {
        % Use full column width
        \includegraphics[width=0.99\linewidth]{figures/competitive_ratio_clustering.pdf}
    }
    \caption{Competitive ratio of our algorithm for different budgets $b$ and clustering approaches evaluated over $20$ simulations of donor arrivals from January to June 2019. The shaded region represents the standard deviation. Compared to $b=1$ and the \emph{status quo}, all results are statistically significant ($p<0.05$, Wilcoxon signed-rank test).}
    \label{fig:competitive_ratio}
\end{figure}
\section{Conclusions}

We introduced near-optimal online matching algorithms via a coarsening-based $b$-matching approach. Our algorithms are suitable for domains where offline nodes can be effectively partitioned into capacitated clusters. Although coarsening typically implies a loss of fidelity, we find that grouping nodes into capacitated clusters improves both theoretical guarantees and practical performance, even as the error of the clustering grows large. Coarsening also simplifies the offline planning phase, making the approach highly effective for domains where scalability is paramount, such as Internet advertising.

Another strength of our approach is its modularity. \Cref{alg:clustered_matching} is agnostic to the clustering subroutine. The framework benefits directly from advancements in unsupervised learning: utilizing better capacitated clustering algorithms may further improve the competitive ratio.

We applied our coarsening-based framework to heart transplant allocation. Ours are, to our knowledge, the first theoretically-grounded policies for this domain. Furthermore, experimentally based on real UNOS data, our algorithms achieve a competitive ratio of 0.91, drastically higher than non-capacitated stochastic matching's 0.63 and the US \emph{status quo} policy's 0.51. 

While we focused on maximizing efficiency (\textit{i.e.}, population survival), our approach is amenable to alternative metrics, including fairness. The weights themselves can be altered to prioritize disadvantaged groups~\citep{Dickerson2015:Futurematch}. The offline planning phase can also enforce fairness measures through additional constraints. 
%Ensuring fair outcomes is an important direction for future work. 

\section*{Impact statement}

Our work introduces a methodology for optimizing heart transplant allocation through a coarsening-based matching framework. While our results demonstrate the potential to significantly extend population survival, the deployment of such algorithmic tools in high-stakes healthcare settings entails significant risks and ethical responsibilities. We identify a number of technical and ethical limitations of our work that impact the safety of deployment.

\paragraph{Data bias and predictive uncertainty} We rely on historical data from the UNOS registry. This data reflects decades of clinical practice and it inevitably contains systemic biases related to healthcare access, candidate selection, and historical inequities. These biases can in turn affect the predictive models underlying our edge weights and risks amplifying inequity. In addition, determining patient outcomes, such as graft survival after a transplant, is a challenging predictive task. Unobserved confounders not present in the registry may influence actual outcomes.  Despite the coarsening-based $b$-matching approach being mathematically agnostic to how the weights are set, deploying any such algorithm relying on imperfect weights can have unintended consequences. Future improvements in survival prediction and bias mitigation are essential for this application.

\paragraph{Ethics of aggregation} The coarsening-based approach also introduces a tradeoff between aggregate performance, which comes at the cost of individual fairness. When clustering patients, we obtain a single representative profile for the group. We lose the nuances of each patient's specific case. While this improves theoretical and simulated performance, it risks benefiting the average patient at the expense of outliers. Our approach also seeks to maximize the population life years gained, which we acknowledge is just one of many possible objectives for the application. We do not necessarily argue that utilitarian maximization is the sole or most appropriate metric for organ allocation.

\paragraph{System dynamics and safety} Our simulation makes simplifying assumptions regarding system dynamics. We assume that the set of waiting patients remains static over the planning horizon. While we show in \Cref{sec:waitlist_stability} that the distribution of patients on the waitlist is relatively stable, this abstraction ignores a patient's individual medical progression. The offline linear program also assumes that future donor arrival rates will mirror historical distributions. Addressing these non-stationary dynamics is an important direction for future research. Deploying a stationary policy in a non-stationary world could lead to brittle decision making if not carefully managed. As a partial solution to this, we suggested metrics of distribution shift that can be used to trigger reoptimization of the dispatch plan.

%Consequently, we envisage this framework potentially functioning as a decision-support tool. It is designed to augment, not replace, the judgment of clinical experts, who must ultimately bridge the gap between algorithmic recommendations and the evolving, individualized needs of patients on the waitlist.
\section*{Acknowledgments}

Tuomas Sandholm and his PhD students Ioannis Anagnostides and Itai Zilberstein are supported by NIH award A240108S001, the Vannevar Bush Faculty Fellowship ONR N00014-23-1-2876, and National Science Foundation grant RI-2312342. Itai Zilberstein is also supported by the NSF Graduate Research Fellowship Program under grant DGE2140739. Arman Kilic is supported by NIH RO1 grant 5R01HL162882-03 which contributed to the funding for completion of this project. Arman Kilic is a speaker and consultant for Abiomed, Abbott, 3ive, and LivaNova, and founder and owner of QImetrix. All additional authors have no financial relationships to disclose. Any opinions, findings, and conclusions or recommendations expressed in this material are those of the author(s) and do not necessarily reflect the views of the funding agencies. We are indebted to Carlos Martinez from UNOS for answering numerous questions.

\bibliography{refs}

\clearpage
\appendix

\section{Further prior work on online matching}
\label{sec:related-matching}

There has been a considerable amount of work on online matching problems, primarily motivated by Internet advertising applications~\citep{Huang24:Online,Mehta10:Online}. The setting here is as follows. There is an underlying bipartite graph $G = (U, V, E)$, where $U$ is the set of offline nodes and $V$ is the set of nodes arriving online. In the basic version of the problem, each offline node $u \in U $ can be matched to at most one online node. As we highlighted earlier in the main body, closely related to our approach is the \emph{$b$-matching} generalization of the problem, in which a node $u \in U$ can be matched to up to $b \in \N$ nodes in $V$~\citep{Kalyanasundaram00:Optimal,Brubach16:New,Alaei12:Online,Devanur09:Adwords}. %In the context of Internet advertising, which has driven much of this research, $U$ represents the set of \emph{advertisers}, while $V$ is the stream of \emph{impressions}. The edge set $E$ captures the compatibility, where an edge $(u, v)$ exists if advertiser $u$ is interested in bidding on impression $v$. In practical settings, to ensure market thickness and computational tractability, automated abstraction methods are often employed to bundle inventory or segments, effectively grouping distinct instances into larger $b$-matching constraints~\citep{Walsh2010:Automated,Peng2016:Scalable}.

The foundational work on online matching was by~\citet{Karp90:Optimal}, who gave a randomized algorithm that attains the optimal competitive ratio of $1 - 1/e$ even against an adversarial arrival sequence. This was subsequently generalized to the \emph{adwords} problem by~\citet{Mehta07:Adwords}, and there has been much follow up work even since. The premise of having an adversarial arrivals is overly pessimistic in our setting, wherein we have ample historical data concerning donor organ arrivals.

Of more relevance to us is the \emph{stochastic, known, i.i.d.} model of arrivals. Here, the bipartite graph is known beforehand and each arriving node $v$ is drawn (with replacement) from a known distribution on $V$. This model is motivated by the fact that, in practice, one often has more background data concerning the impressions and can estimate the frequency of upcoming arrivals.  In the edge-weighted model~\citep{Feldman09:Online}, one can think of each advertiser as gaining a certain revenue---given by the weight of the edge---for being matched to a particular type of impression. A special case of this model is vertex-weighted online matching~\citep{Aggarwal11:Online}, where weights are only associated with the advertisers. This data-driven perspective aligns with the optimize-and-dispatch architecture found in ad exchanges~\citep{Parkes2005:Optimize}, and similar stochastic optimization techniques have been adapted to handle dynamic arrivals in kidney exchange~\citep{Awasthi2009:Online}.

The specific setting of $b$-matching has a long history (\emph{e.g.},~\citealp{Kalyanasundaram00:Optimal,Alaei12:Online,Devanur09:Adwords}), but to our knowledge, such ideas have not been employed in the context of organ allocation before. Our approach in~\Cref{sec:theory} makes use of the framework of~\citet{Brubach16:New} which can also be used under stochastic rewards.
\section{Omitted proofs}
\label{sec:proofs}

\setcounter{table}{0}
\renewcommand{\thetable}{A\arabic{table}}

\setcounter{algorithm}{0}
\renewcommand{\thealgorithm}{A\arabic{algorithm}}

\setcounter{figure}{0}
\renewcommand{\thefigure}{A\arabic{figure}}

For completeness, we begin by providing the proof of~\Cref{thm:bmatching} following the argument of~\citet{Brubach16:New}.

\bmatching*

\begin{proof}
Let $A_t$ denote the number of times vertex $u$ has been matched at the start of round $t$. Define $B(u,t)$ as the event that vertex $u$ is \emph{safe} at the beginning of round $t$, which is to say that $A_t \leq b - 1$.

For any edge $e = (u, v) \in E$, let $X_e$ represent the total number of times edge $e$ is matched throughout the $T$ rounds. The expected value can be expressed as
\begin{equation}
\E[X_e] = \sum_{t=1}^{T} \pr[B(u,t)] \cdot \frac{r_v}{T} \cdot \frac{f_e}{r_v} \cdot p_e = \frac{f_e p_e}{T} \sum_{t=1}^{T} \pr[A_t \leq b-1].
\end{equation}

We proceed to establish an upper bound on $\Pr[A_t \geq b]$. For each round $1 \leq \tau \leq t$, define $Z_\tau$ as the indicator random variable for the event that $u$ is matched during round $\tau$. Consequently, $A_{t+1} = \sum_{\tau=1}^{t} Z_\tau$. For each $\tau$, we have
\begin{equation*}
\E[Z_\tau] \leq \sum_{v} \frac{r_v}{T} \cdot \frac{f_{u,v}}{r_v} \cdot p_{u,v} \leq \frac{b}{T}.
\end{equation*}

It follows that for any $t \leq T(1-\delta)$ with $0 < \delta < 1$, $\E[A_{t+1}] \leq (1-\delta)b$. Applying the Chernoff-Hoeffding concentration inequality yields
\begin{equation*}
\pr[A_{t+1} \geq b] \leq e^{-b\delta^2/3}.
\end{equation*}
Therefore,
\begin{align*}
\mathbb{E}[X_e] &= \frac{f_e p_e}{T} \sum_{t=1}^{T} \Pr[A_t \leq b-1] \\
&\geq \frac{f_e p_e}{T} \sum_{t=1}^{T(1-\delta)} \left(1 - e^{-b\delta^2/3}\right) \\
&= f_e p_e (1-\delta)\left(1 - e^{-b\delta^2/3}\right).
\end{align*}
For any $\epsilon > 0$, selecting $\delta \defeq b^{-1/2 + \epsilon}$ completes the proof since $\sum_{e \in E} w_e f_e p_e \geq \E[\opt]$, where $\E[\opt]$ is the hindsight optimal expected value.
\end{proof}

We continue with the proof of~\Cref{thm:clustered_bmatching}.

\clusteredbmatching*

\begin{proof}
Let us fix the size of the cluster $b$. To bound the error we proceed as follows. Let $E^*$ be the set of edges selected by the optimal, omniscient policy. By definition, the expected value of this matching is 
\begin{equation*}
    \E[\opt] = \sum_{(u_k, v) \in E^*} w_{u_k, v} p_{u, v}.
\end{equation*}
By feasibility, we know that every node $u_k$ can only be matched once. Now, let $\E[\overline{\opt}]$ be the optimal in hindsight expected value under the representative weights. That is, $\overline{\opt}$ concerns the induced $b$-matching problem in which each cluster $\{u_1, \dots, u_b \}$ is replaced by a node $u$ with capacity $b$. By optimality, it follows that 
\[
\E[\overline{\opt}] \geq \sum_{(u_k, v) \in E^*} \bar{w}_{u, v} p_{u, v}
\]
since $E^*$ induces a feasible solution for the $b$-matching problem. By~\Cref{ass:cluster_relative}, we know that $\bar{w}_{u, v} \geq (1 + \delta(b))^{-1} w_{u_k, v}$ for any $k \in [b]$. So, $\E[\overline{\opt}] \geq (1 + \delta(b))^{-1} \E[\opt]$.

Similarly, let $\overline{\alg}$ be~\Cref{alg:bmatching} under the representative weights. The proposed algorithm $\alg$ proceeds by matching an incoming node $v$ to a cluster $u$ per~$\overline{\alg}$, whereupon it picks any unmatched node $u_k$ (if any) within that cluster. By~\Cref{ass:cluster_relative}, this ensures that
\begin{equation*}
    \E[\overline{\alg}] \leq \frac{1}{1 - \delta(b)} \E[\alg].
\end{equation*}

Finally, we know from~\Cref{thm:bmatching} that relative to the $b$-matching instance induced by the representative weights,
\begin{equation*}
 \frac{\E[\overline{\alg}]}{\E[\overline{\opt}]} \geq \alpha(b).
\end{equation*}
Combining these bounds, the claim follows.
\end{proof}

More broadly, suppose that $(1 - \delta_u(b)) \bar{w}_{u, v} \leq w_{u_k, v} \leq (1 - \delta_u(b))^{-1} \bar{w}_{u, v}$ for each cluster $u \in \mathcal{U}$; that is, we allow the clusters to have different errors. We define the \emph{average weighed error} with respect to the optimal solution as
\begin{equation}
    \label{eq:aver-weigh-error-opt}
    \Delta_\opt = \frac{ \sum_{(u_k, v) \in E^*} \delta_u w_{u_k, v} p_{u, v}}{ \E[\opt]}.
\end{equation}
This acts as a weighted average where the weight of a cluster's error $\delta_u(b)$ is proportional to the cluster's share of the optimal value $\E[\opt]$. As a result, the error is driven by cluster variance and cluster importance: the performance degrades only if the optimal policy relies heavily on clusters with high approximation errors. This definition allows us to write $\E[\overline{\opt}] \geq (1 - \Delta_\opt) \E[\opt]$.

Similarly, if we define
\begin{equation}
    \label{eq:aver-weigh-error-alg}
    \Delta_\alg = \frac{ \sum_{(u_k, v) \in E_{\alg}} \delta_u \bar{w}_{u, v} p_{u, v} }{\E[\overline{\alg}]},
\end{equation}
it follows that $\E[\alg] \geq (1 - \Delta_\alg) \E[\overline{\alg}]$. Above, $E_{\alg}$ denotes the (random) set of edges chosen by $\alg$. $\Delta_\opt$ and $\Delta_\alg$ are \emph{a posteriori} error measures, as they are contingent on the specific edges selected by $\alg$ and $\opt$. Following the proof of~\Cref{thm:clustered_bmatching}, we can obtain the following result.

\begin{theorem}
\label{thm:clustered_avg}
Consider a weighted online stochastic matching problem. If~$\mathsf{ALG}$ is set to~$\mathsf{SM}_b$ on the $b$-matching instance induced by the representative weights,
\begin{equation*}
\frac{\E[\mathsf{ALG}]}{\E[\opt]} \geq \alpha(b) ( 1 - \Delta_\opt) (1 - \Delta_\alg),
\end{equation*}
where $\Delta_\opt$ and $\Delta_\alg$ are defined in~\eqref{eq:aver-weigh-error-opt} and~\eqref{eq:aver-weigh-error-alg}, respectively.
\end{theorem}

We proceed with the proof of~\Cref{thm:bad_clusters}.

\badclusters*

\begin{proof}
The representative weights on edges incident $\Ubad$ will be defined as zero. To extend the analysis of~\Cref{thm:clustered_bmatching} in this setting, we proceed as follows. We write $\E[\alggood]$ and $\E[\algbad]$ to denote the expected weight obtained by $\alg$ restricted on $\Ugood$ and $\Ubad$, respectively, and similarly for $\overline{\alg}$. We have
\begin{equation*}
    \E[\overline{\alg}] = \E[\baralggood] + \E[\baralgbad].
\end{equation*}
From~\Cref{thm:clustered_bmatching}, we know that $\E[\overline{\alg} ] \geq \alpha(b) \overline{\opt}$. Furthermore, $\E[\baralgbad] = 0$ since the representative weights on edges incident to $\Ubad$ have been defined as zero. In other words, $\E[\baralggood] \geq \alpha(b) \E[\overline{\opt}]$.

Since weights are nonnegative, $\E[\alg] \geq \E[\alggood] \geq (1 - \delta(b)) \E[\baralggood] \geq \alpha(b) (1 - \delta(b)) \E[\overline{\opt}]$, where we used the fact that $\Ugood$ contains clusters with relative error bounded by $\delta(b)$ (\Cref{ass:most_cluster_good}). 

Finally, we have $\E[\optgood] = \E[\opt] - \E[\optbad] \geq (1 - \rho) \E[\opt]$, by~\Cref{ass:most_cluster_good}, and $\E[\overline{\opt}] \geq \E[\baroptgood] \geq (1 + \delta(b))^{-1} \E[\optgood]$. As a result,
\begin{equation*}
    \E[\overline{\opt}] \geq \frac{1 - \rho(b)}{1 + \delta(b)} \E[\opt].
\end{equation*}
Combining these bounds, the statement follows.
\end{proof}

Similar reasoning establishes the following result.

\begin{theorem}
    \label{thm:bounded-est-most}
Consider a weighted online stochastic matching problem with $\eta$-bounded estimation error (\Cref{ass:imprecise_weights}) satisfying \Cref{ass:most_cluster_good} relative to the given weights. Then there exists an algorithm $\mathsf{ALG}$ such that
\begin{equation*}
\frac{\E[\mathsf{ALG}]}{\E[\opt]} \geq (1 - 2\eta) \max_{b \in \mathcal{B}} \alpha(b) ( 1-\rho(b)) ( 1 - 2 \delta(b)).
\end{equation*}
\end{theorem}

\section{Capacitated clustering algorithms}
\label{sec:appendix-clustering}

Standard clustering algorithms like $k$-means or hierarchical clustering do not enforce minimum cluster sizes. While minimum-cost flow formulations can enforce these constraints~\citep{Bradley2000:Constrained}, they require iterative solving and remain computationally prohibitive for datasets of this magnitude.

We adapt these standard algorithms by applying a two-phase heuristic repair strategy~\citep{Mulvey1984:Solving,Ahmadi2005:Greedy} that enforces a minimum capacity $b$ while preventing the formation of excessively large clusters. We initialize the number of clusters to $\lfloor |U|/b \rfloor$. After the initial unconstrained clustering, we apply a greedy repair procedure. 

First, in the \textit{merge phase}, we identify all clusters with cardinality less than $b$ and sort them by ascending size. Starting with the smallest, we iteratively merge each violating cluster into its nearest neighbor based on the centroid distance. We prioritize merging two under-capacity clusters together before absorbing a small cluster into a sufficiently large one. Second, in the \textit{split phase}, we identify clusters that are excessively large (\emph{e.g.}, $|u| \ge 2b$) and decompose them using the recursive bisection algorithm described below. 

We use the standard $k$-means algorithm to leverage its linear scalability, followed by the repair step to enforce the capacity constraint. Agglomerative clustering works bottom-up, iteratively merging the two closest clusters until the desired number of clusters is reached, at which point the repair process begins. 

\textit{Recursive bisection}~\citep{Karypis2000:Comparison} offers another efficient heuristic approximation by using a top-down divisive strategy. Recursive bisection clustering reduces the problem to a sequence of $2$-means clustering. When the $2$-means split fails to satisfy the cardinality requirement, we default to geometric partitioning. We project all points onto the vector connecting the two centroids and split along this axis at the index closest to the median such that the minimum size threshold is respected.  This yields an approximate solution with a time complexity of roughly $O(N \log(N/b))$. Recursive bisection is highly efficient even for massive patient registries. 

\section{Error measures and selecting a suitable capacity}
\label{sec:appendix-hcr}

We define the error metrics relative to each node offline $u$. Let $w_{u,v}$ denote the edge weight between $u$ and $v$, and let $\bar{w}_{c(u),v}$ denote the representative weight of the cluster $c(u)$ assigned to $u$.

First, we calculate the \textit{mean absolute error (MAE)} for a node $u$ by averaging the deviation across all online nodes $v \in V$:
\begin{equation*}
    \text{MAE}_u = \frac{1}{|V|} \sum_{v \in V} | w_{u,v} - \bar{w}_{c(u),v} |.
\end{equation*}

To facilitate comparison, we normalize this error by the global maximum utility observed in the dataset, $w_{\max} = \max_{u,v} w_{u,v}$. We normalize by the maximum rather than the mean to ensure stability, as the mean utility can be near-zero. The \textit{normalized mean absolute error (NMAE)} is:
\begin{equation*}
    \text{NMAE}_u = \frac{\text{MAE}_u}{w_{\max}}.
\end{equation*}

The mean and maximum NMAE are defined as:
\begin{equation*}
    \text{NMAE}_{\textsc{mean}} = \frac{1}{|U|} \sum_{u \in U} \text{NMAE}_u, \quad \text{NMAE}_{\max} = \max_{u \in U} \text{NMAE}_u.
\end{equation*}

To select the optimal $b$ in the offline phase, we need to bridge the gap between the pessimistic theoretical bound and empirical reality. We introduce a \textit{heuristic competitive ratio} ($\HCR$) to approximate the relative competitive ratio of the online algorithm. The heuristic competitive ratio estimates the relative performance using the formula 
\begin{equation*}\label{eq:hcr}
    \HCR(b) \coloneqq \left(1-\frac{1}{\sqrt{b}}\right)\cdot \left(1 -\text{NMAE}_{\max}\right).
\end{equation*}

The $\HCR$ is an overly conservative estimate of the competitive ratio (due to lower bounds and omission of constant factors), and therefore only shows relative values for different capacities $b$. It can be used to identify, in advance, a suitable minimum cluster size.  We show the $\HCR$ in \Cref{fig:hcr} for different clustering algorithms and capacities $b$.  The $\HCR$ highlights the tradeoff between larger capacities and the error of the clustering. We find that a minimum cluster size of $10$ to $30$ is optimal. We see in online simulations that these values of $b$ indeed result in the strongest competitive ratios. An alternative way to tune $b$ is to simply evaluate outcomes and select the capacity with the strongest competitive ratio in offline simulations. 

\begin{figure}[h]
    \centering
    \includegraphics[width=0.5\linewidth]{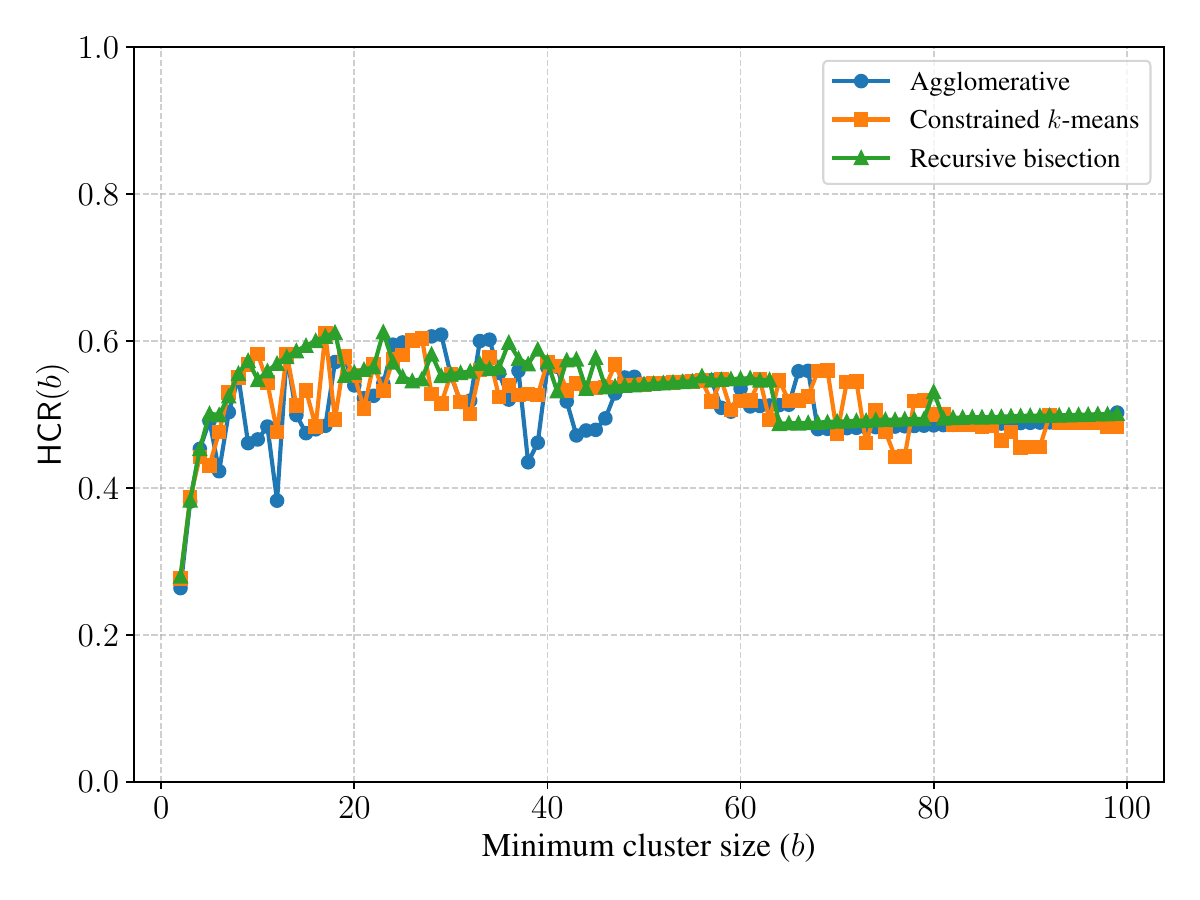}
    \caption{Heuristic competitive ratio of clustering algorithms across cluster sizes.}
    \label{fig:hcr}
\end{figure}

\section{Run time of algorithms}
\label{sec:appendix-runtime}

\begin{figure*}[h!]
    \centering
    % (a) clustering
    \begin{subfigure}[b]{0.45\textwidth}
        \centering
        \includegraphics[width=\linewidth]{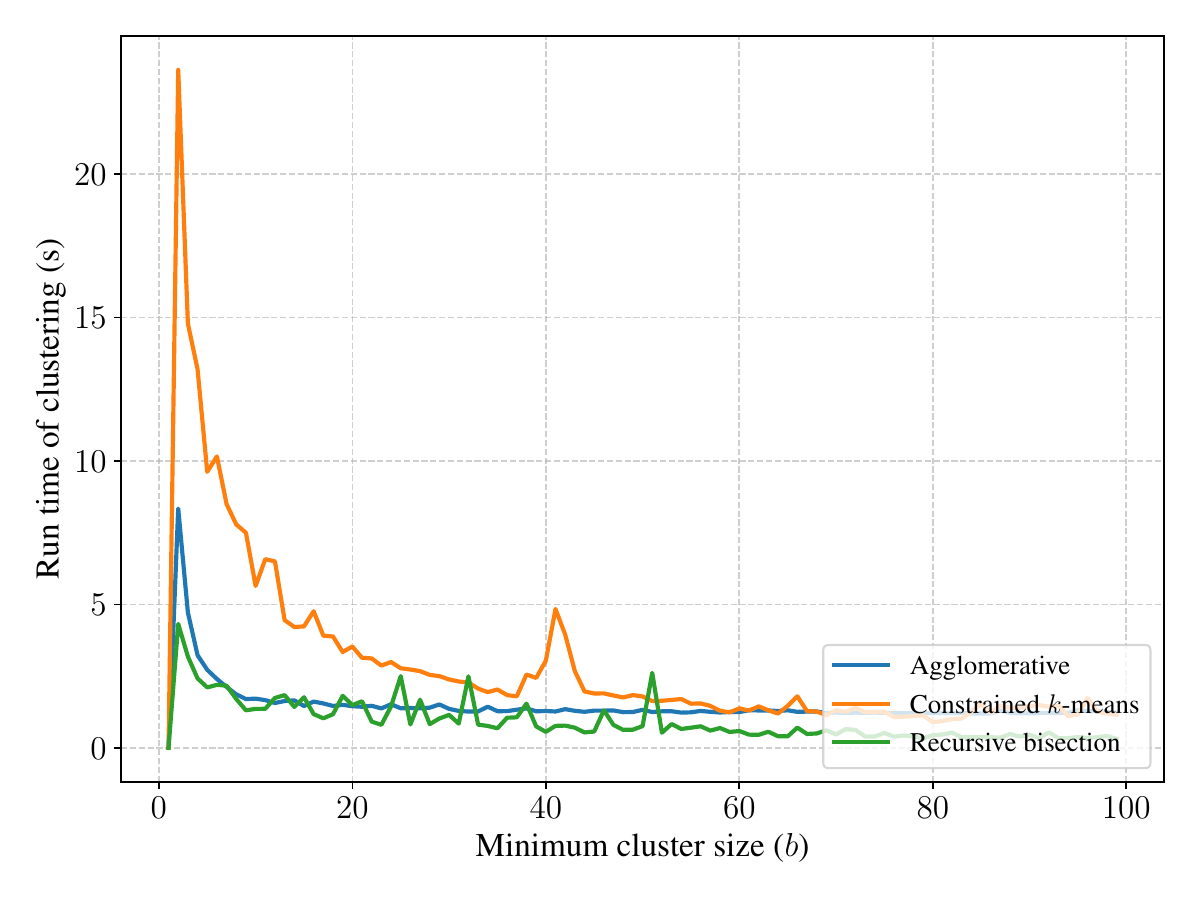}
        \caption{Run time of capacitated clustering algorithms (s).}
        \label{fig:runtime_a}
    \end{subfigure}
    \hfill
    \begin{subfigure}[b]{0.45\textwidth}
        \centering
        \includegraphics[width=\linewidth]{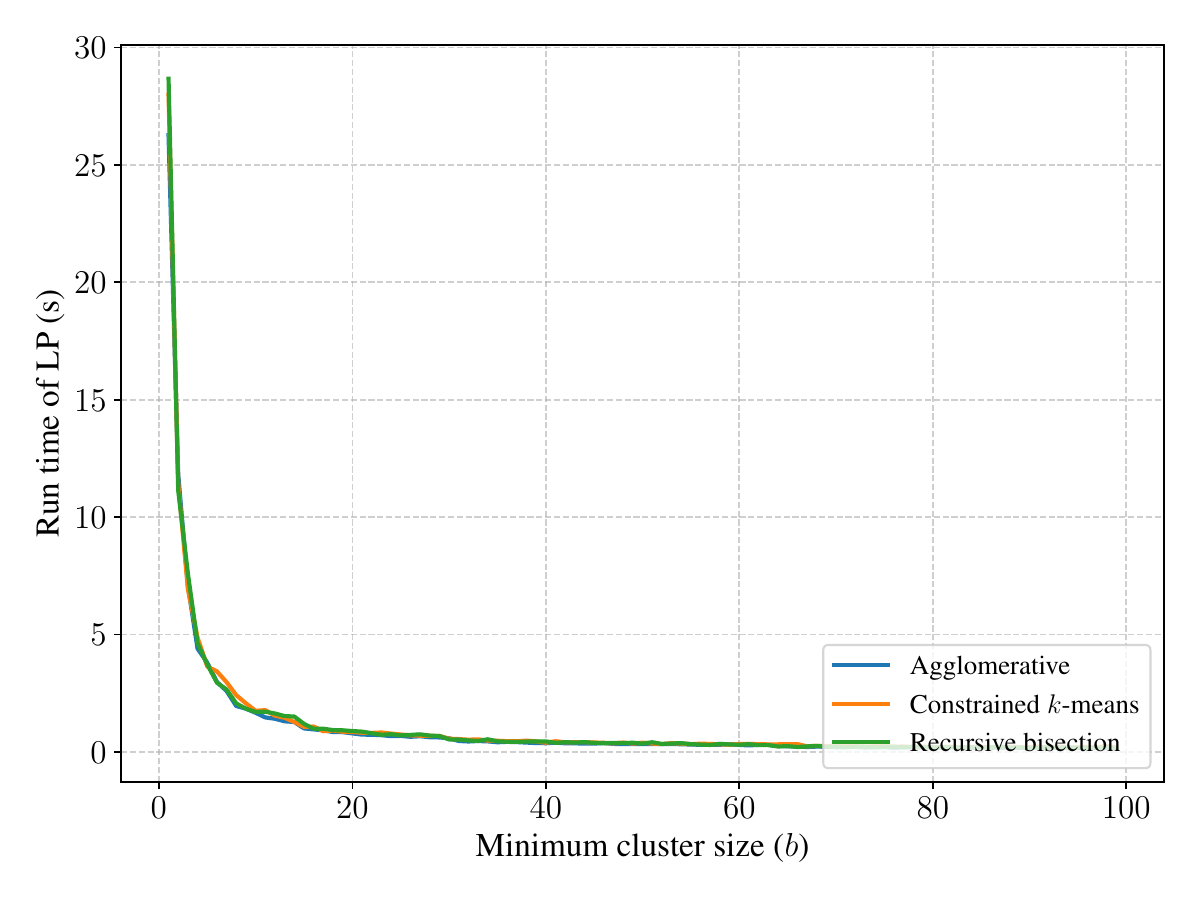}
        \caption{Run time of offline linear program (s).}
        \label{fig:runtime_b}
    \end{subfigure}
    
    \caption{Run time of the components of our algorithm across cluster sizes. Panel (a): runtime of capacitated clustering algorithms. 
    Panel (b): runtime of offline linear program.}
    \label{fig:runtime}
\end{figure*}

We analyze the computational cost of our framework in \Cref{fig:runtime}. The total run time consists of the clustering step and solving the offline linear program. All experiments are conducted on a machine equipped with an Apple M4 Pro chip (14-core CPU) and 24 GB of unified memory. 

For $b=1$, the process is entirely dominated by the linear program, which takes approximately $30$ seconds. However, as $b$ increases, the number of decision variables in the LP reduces proportionally, driving the solve time down to under $2$ seconds for $b \ge 20$, an order of magnitude reduction. 

We observe an initial spike in run time for small capacitated clustering, particularly for constrained $k$-means. This is likely due to the difficulty of enforcing strict capacity constraints when the target clusters are small and numerous, requiring more iterative repairs. As $b$ increases, these constraints become easier to satisfy, and the clustering run time stabilizes below $3$ seconds.

For $b \ge 20$, the end-to-end process finishes in under $5$ seconds, an $80$--$90\%$ reduction compared to the $b=1$ baseline. Our coarsening-based approach not only improves solution quality, it simultaneously reduces the computational cost substantially.

\section{Resampling to avoid discarding}
\label{sec:appendix-resampling}

\begin{figure}[h]
    \centering
    \includegraphics[width=0.5\linewidth]{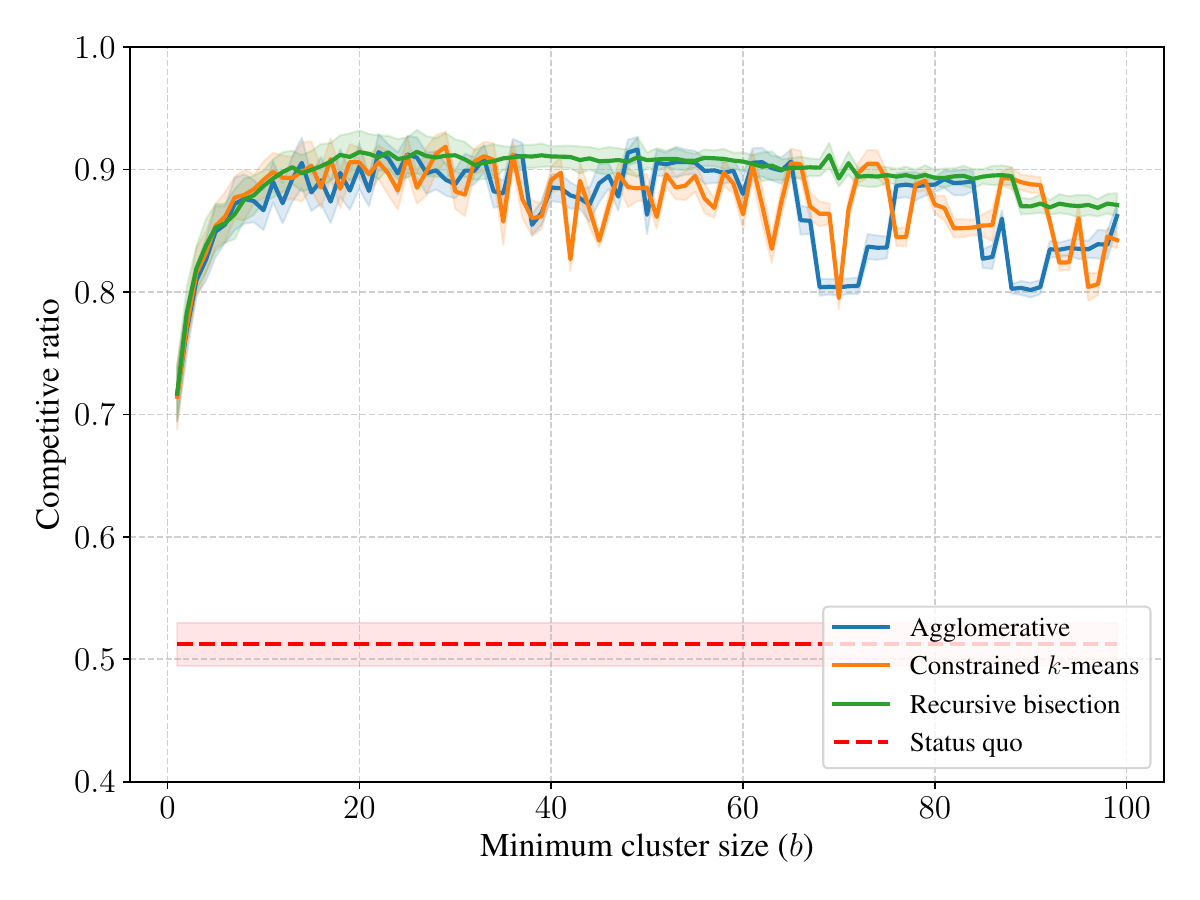}
    \caption{Competitive ratio of our algorithm when allowing resampling to avoid donor discarding for different budgets $b$ and clustering approaches evaluated over $20$ simulations of donor arrivals from January to June 2019. The shaded region represents the standard deviation.}
    \label{fig:discard}
\end{figure}

We previously evaluate our algorithm using a conservative dispatch that discards an organ if the offline solution does not match it (\textit{i.e.}, all available patients have been matched). However, in real-world transplant systems, the demand far outweighs the supply. Therefore, discarding a viable organ is almost never desired. 

We evaluate our algorithm using a dispatch that re-routes an organ that would otherwise be discarded. In the re-routing variant of our algorithm, if an organ is not initially matched, the algorithm resamples a new node from the LP strategy distribution, using the randomized policy conditional on the unavailability the previous choices. 

We present the results of this variant in \Cref{fig:discard}. The competitive ratio for the $b=1$ baseline improves compared to the strict discarding case, rising to approximately 0.72. However, the competitive ratio of the coarsening-based framework remains dominant. As $b$ increases, the competitive ratio rises to over 0.9. This shows that the benefits of clustering persist even when the baseline is allowed to recover from routing failures.

\section{Larger planning horizons and greedy selection}
\label{sec:appendix-greedy}

\begin{figure}[h!]
    \centering
    \includegraphics[width=0.5\linewidth]{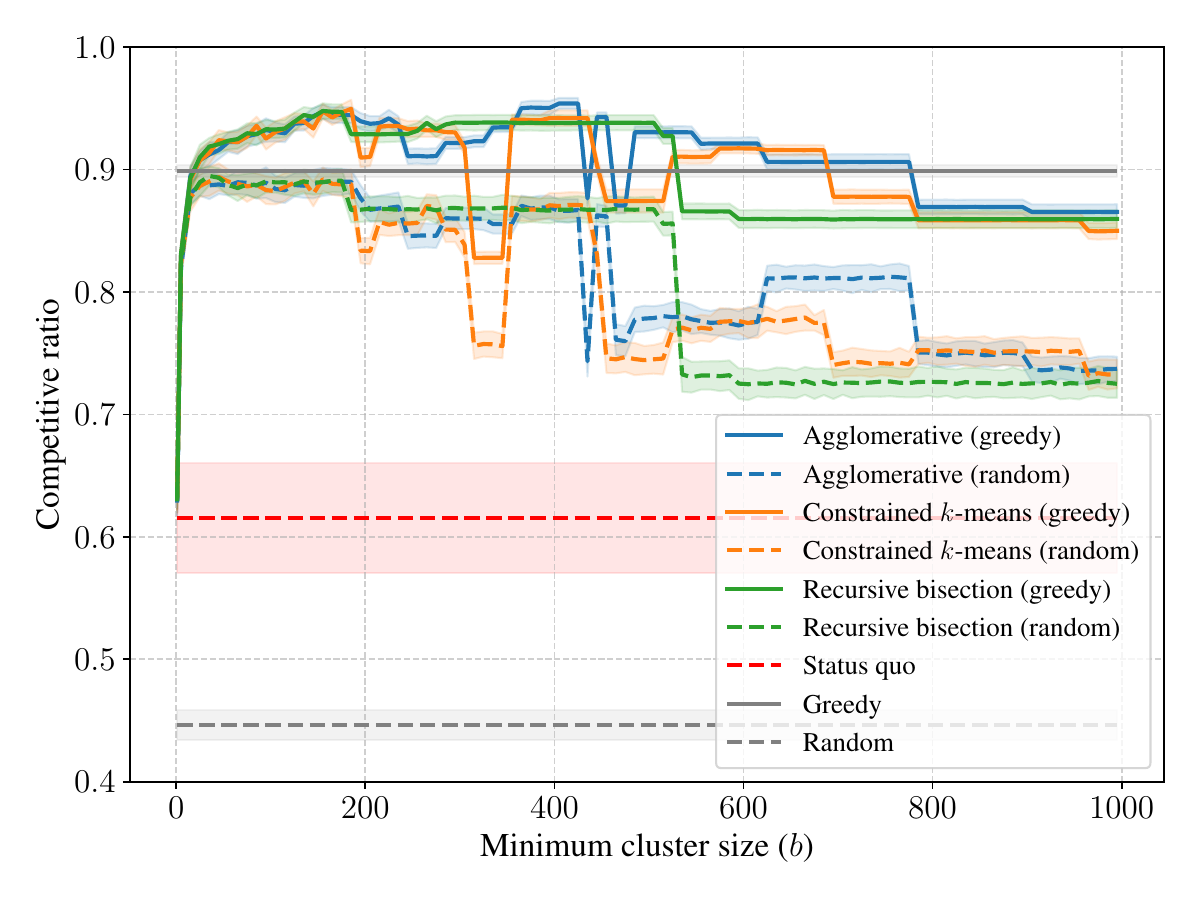}
    \caption{Competitive ratio of our algorithm for different budgets $b$, clustering approaches, and intra-cluster selection policies evaluated over $20$ simulations of $2,500$ donor arrivals from January to June 2019. The shaded region represents the standard deviation.}
    \label{fig:cr-greedy}
\end{figure}

In general, myopic decision-making has been shown to be effective for heart transplant allocation~\citep{Anagnostides25:Policy}. We test our method against a greedy baseline over an extended planning horizon of 2,500 donor arrivals. This volume approximates a market-clearing scenario where the supply of donors nearly matches the demand from the waitlist. While this diverges from the reality of organ allocation, it provides a valuable benchmark for other applications where supply and demand are balanced, and illustrates the robustness of our approach. 

\Cref{fig:cr-greedy} illustrates the performance of our algorithm using both random and greedy policies to select a node within a cluster. The greedy intra-cluster policy selects the patient within the assigned cluster maximizing the immediate edge weight. We observe that the global greedy baseline performs well, achieving a competitive ratio of approximately 0.9. This is comparable to our coarsening-based approach when using random intra-cluster selection.

However, combining the coarsening-based approach  with greedy intra-cluster selection yields superior performance, pushing the competitive ratio above 0.95. As the minimum cluster size increases to $|U|$, this strategy effectively converges to the global greedy policy. Yet at modest $b$, it consistently outperforms the global greedy baseline. Remarkably, even with random intra-cluster selection, coarsening outperforms the $b=1$ baseline and random matching at large $b$. This shows that coarsening is a robust strategy regardless of the capacity of the cluster. 

\section{Omitted tables and figures}

\begin{table*}[h]
\centering
\resizebox{\linewidth}{!}{
\begin{tabular}{r c c c c c c c c}
\toprule
\multirow{2}{*}{\textbf{Clusters ($k$)}} & \multicolumn{2}{c}{\textbf{Type O} ($n=39,341$)} & \multicolumn{2}{c}{\textbf{Type A} ($n=27,665$)} & \multicolumn{2}{c}{\textbf{Type B} ($n=8,020$)} & \multicolumn{2}{c}{\textbf{Type AB} ($n=1,664$)} \\
\cmidrule(lr){2-3} \cmidrule(lr){4-5} \cmidrule(lr){6-7} \cmidrule(lr){8-9}
 & \textbf{ASE} & \textbf{Var.} & \textbf{ASE} & \textbf{Var.} & \textbf{ASE} & \textbf{Var.} & \textbf{ASE} & \textbf{Var.} \\
\midrule
10   & 17.70 & 47.99\% & 17.60 & 48.13\% & 17.50 & 48.41\% & 16.40 & 50.40\% \\
50   & 7.96  & 76.60\% & 7.73  & 77.25\% & 7.61  & 77.61\% & 7.05  & 78.63\% \\
100  & 6.01  & 82.33\% & 5.68  & 83.29\% & 5.56  & 83.65\% & 4.76  & 85.59\% \\
250  & 3.92  & 88.47\% & 3.70  & 89.11\% & 3.47  & 89.80\% & 2.24  & 93.21\% \\
500  & 2.80  & 91.77\% & 2.56  & 92.48\% & 2.23  & 93.44\% & 0.80  & 97.59\% \\
1000 & 1.88  & 94.46\% & 1.66  & 95.12\% & 1.16  & 96.58\% & 0.10  & 99.71\% \\
\bottomrule
\end{tabular}}
\caption{Donor clustering across blood type partitions. We track the average squared error (ASE) and explained variance (Var.) as the number of representative types ($k$) increases.}
\label{tab:donor_clustering_results}
\end{table*}

\section{\emph{Status quo} heart transplant allocation policy}
\label{appendix:statusquo}

The US adult heart transplant allocation policy is a hierarchical system that prioritizes primarily based medical urgency status, blood type compatibility, and geographic proximity. It comprises 68 priority tiers, where 1 represents the highest priority. Each tier is defined by a combination of i) medical status ranging from status 1 (highest urgency) to status 6; ii) blood type compatibility (where \emph{secondary} blood compatibility applies only between type O donors and type A or type AB patients); and geographic proximity between donor and recipient location. \Cref{tab:priority_tiers} contains the detailed description of each tier.

\begin{table}[htbp]
\footnotesize
\centering
\caption{\emph{Status quo} priority tiers.}
\label{tab:priority_tiers}
\begin{tabular}{clll@{\hspace{1em}}|@{\hspace{1em}}clll}
\toprule
\textbf{Tier} & \textbf{Status} & \textbf{Blood match} & \textbf{Distance (nm)} & \textbf{Tier} & \textbf{Status} & \textbf{Blood match} & \textbf{Distance (nm)} \\
\midrule
1 & 1 & Primary & $\leq 500$ & 35 & 2 & Primary & $\leq 2500$ \\
2 & 1 & Secondary & $\leq 500$ & 36 & 2 & Secondary & $\leq 2500$ \\
3 & 2 & Primary & $\leq 500$ & 37 & 3 & Primary & $\leq 2500$ \\
4 & 2 & Secondary & $\leq 500$ & 38 & 3 & Secondary & $\leq 2500$ \\
5 & 3 & Primary & $\leq 250$ & 39 & 4 & Primary & $\leq 1000$ \\
6 & 3 & Secondary & $\leq 250$ & 40 & 4 & Secondary & $\leq 1000$ \\
7 & 1 & Primary & $\leq 1000$ & 41 & 5 & Primary & $\leq 1000$ \\
8 & 1 & Secondary & $\leq 1000$ & 42 & 5 & Secondary & $\leq 1000$ \\
9 & 2 & Primary & $\leq 1000$ & 43 & 6 & Primary & $\leq 1000$ \\
10 & 2 & Secondary & $\leq 1000$ & 44 & 6 & Secondary & $\leq 1000$ \\
11 & 4 & Primary & $\leq 250$ & 45 & 1 & Primary & Any \\
12 & 4 & Secondary & $\leq 250$ & 46 & 1 & Secondary & Any \\
13 & 3 & Primary & $\leq 500$ & 47 & 2 & Primary & Any \\
14 & 3 & Secondary & $\leq 500$ & 48 & 2 & Secondary & Any \\
15 & 5 & Primary & $\leq 250$ & 49 & 3 & Primary & Any \\
16 & 5 & Secondary & $\leq 250$ & 50 & 3 & Secondary & Any \\
17 & 3 & Primary & $\leq 1000$ & 51 & 4 & Primary & $\leq 1500$ \\
18 & 3 & Secondary & $\leq 1000$ & 52 & 4 & Secondary & $\leq 1500$ \\
19 & 6 & Primary & $\leq 250$ & 53 & 5 & Primary & $\leq 1500$ \\
20 & 6 & Secondary & $\leq 250$ & 54 & 5 & Secondary & $\leq 1500$ \\
21 & 1 & Primary & $\leq 1500$ & 55 & 6 & Primary & $\leq 1500$ \\
22 & 1 & Secondary & $\leq 1500$ & 56 & 6 & Secondary & $\leq 1500$ \\
23 & 2 & Primary & $\leq 1500$ & 57 & 4 & Primary & $\leq 2500$ \\
24 & 2 & Secondary & $\leq 1500$ & 58 & 4 & Secondary & $\leq 2500$ \\
25 & 3 & Primary & $\leq 1500$ & 59 & 5 & Primary & $\leq 2500$ \\
26 & 3 & Secondary & $\leq 1500$ & 60 & 5 & Secondary & $\leq 2500$ \\
27 & 4 & Primary & $\leq 500$ & 61 & 6 & Primary & $\leq 2500$ \\
28 & 4 & Secondary & $\leq 500$ & 62 & 6 & Secondary & $\leq 2500$ \\
29 & 5 & Primary & $\leq 500$ & 63 & 4 & Primary & Any \\
30 & 5 & Secondary & $\leq 500$ & 64 & 4 & Secondary & Any \\
31 & 6 & Primary & $\leq 500$ & 65 & 5 & Primary & Any \\
32 & 6 & Secondary & $\leq 500$ & 66 & 5 & Secondary & Any \\
33 & 1 & Primary & $\leq 2500$ & 67 & 6 & Primary & Any \\
34 & 1 & Secondary & $\leq 2500$ & 68 & 6 & Secondary & Any \\
\bottomrule
\end{tabular}
\end{table}

\end{document}